%% file: prefinal.tex
\documentclass[10pt,twocolumn,letterpaper]{article}

\usepackage[pagenumbers]{wacv}           % arXiv version

\usepackage{graphicx}
\usepackage{amsmath}
\usepackage{amssymb}
\usepackage{amsthm}
\usepackage{enumitem}
\usepackage{booktabs}
\usepackage[table]{xcolor}
\usepackage{url}

\usepackage[pagebackref,breaklinks,colorlinks]{hyperref}
\usepackage[capitalize]{cleveref}
\crefname{section}{Sec.}{Secs.}
\Crefname{section}{Section}{Sections}
\Crefname{table}{Table}{Tables}
\crefname{table}{Tab.}{Tabs.}
\crefname{figure}{Fig.}{Figs.}
\Crefname{figure}{Figure}{Figures}
\crefname{corollary}{Cor.}{Cors.}
\Crefname{corollary}{Corollary}{Corollaries}
\crefname{equation}{Eq.}{Eqs.}

\theoremstyle{plain}
\newtheorem*{theorem}{Theorem}
\newtheorem*{proposition}{Proposition}
\newtheorem{corollary}{Corollary}
\theoremstyle{definition}
\newtheorem*{definition}{Definition}
\newcommand{\thmref}{\hyperref[thm:gauge]{Theorem}}
\newcommand{\propref}{\hyperref[prop:orbit]{Proposition}}

\crefname{section}{Sec.}{Secs.}
\Crefname{section}{Section}{Sections}
\Crefname{table}{Table}{Tables}
\crefname{table}{Tab.}{Tabs.}
\crefname{figure}{Fig.}{Figs.}
\Crefname{figure}{Figure}{Figures}

\theoremstyle{plain}
\newtheorem*{theoremR}{Theorem 1 (restated)}
\newtheorem*{propositionR}{Proposition 1 (restated)}
\newtheorem*{corollaryR}{Corollary 2 (restated)}

\newcommand{\C}{\mathbb{C}}
\newcommand{\Z}{\mathbb{Z}}
\newcommand{\R}{\mathbb{R}}
\newcommand{\E}{\mathbb{E}}
\newcommand{\SN}{S_N}
\DeclareMathOperator*{\argmin}{arg\,min}
\newcommand{\protoname}{\textsc{Met-Sweep}}
\newcommand{\method}{\textsc{JigSync}}

\definecolor{jigviolet}{RGB}{106,45,168}
\newcommand{\ours}[1]{\textcolor{jigviolet}{#1}}
\newcommand{\oursb}[1]{\textcolor{jigviolet}{\textbf{#1}}}
\newcommand{\oursu}[1]{\textcolor{jigviolet}{\underline{#1}}}

\begin{document}

\title{\textbf{\method}: Gauge-Resolved Synchronization for\\
Jigsaw Reassembly under Unknown Piece Orientation}

% \author{Anonymous \confName\ \confYear\ submission\\
% Paper ID \wacvPaperID}
% CAMERA-READY author block:
\author{Soham Pahari \quad Antik Aich Roy \quad Ujjwal Bhattacharya\\
Indian Statistical Institute, Kolkata\\
\url{https://sohamrsch.github.io/jigsync}}

\maketitle

%%%%%%%%% ABSTRACT
\begin{abstract}
   Square jigsaw reassembly requires recovering the spatial arrangement
   of shuffled fragments from their visual content and pairwise
   relationships. While recent studies have made substantial progress,
   existing benchmarks typically assume that all fragments are provided
   upright, reducing reassembly to a permutation problem. We study the
   generalized problem in which each fragment may also have gone through
   an unknown rotation. For this setting we establish
   gauge-unobservability theorem: the minimum of the weighted
   least-squares objective is exactly invariant under a uniform global
   rotation of arbitrary magnitude, so no residual-based criterion can
   recover the global orientation. The theorem further identifies how the
   issue of global orientation can be resolved: an orientation anchor
   estimated from the content of a single fragment, lying outside its
   scope, suffices. To address the above we propose    \method\ attains \ours{$63.8\%$} and \ours{$31.8\%$} absolute accuracy (AA)
   on GAP-3 and GAP-5, the highest reported on both, while additionally
   recovering a rotation per piece that neither benchmark requires. We release \protoname, a degradation protocol that sweeps
   shape, erosion, photometry, grid size and rotation independently. 
\end{abstract}

%%%%%%%%% BODY
\section{Introduction}
\label{sec:intro}

Square jigsaw reassembly is a structured spatial reasoning problem that
requires recovering the original arrangement of shuffled fragments from their
visual content and pairwise relationships. Most existing methods share a
common simplification: fragments are provided upright, and orientation is
either fixed by construction or used only as a training-time
augmentation, so the task actually solved and scored is a permutation
problem over $N$ slots, $|\SN| = N!$.

In real-world reassembly, fragments may arrive in any orientation. A
potsherd, a shredded document strip, or a fresco fragment has no
predetermined orientation until the reconstruction supplies one. Each piece then carries a second decision, its placement and its
orientation, and these compound across fragments to expand the search
space from $N!$ to $4^N N!$, the size of the wreath product $\Z_4 \wr
\SN$. At
$N = 9$ this is $262{,}144\times$ the permutation-only setting; at
$N = 25$ the factor exceeds $10^{15}$.

\begin{figure}[t]
  \centering
  \includegraphics[width=\linewidth]{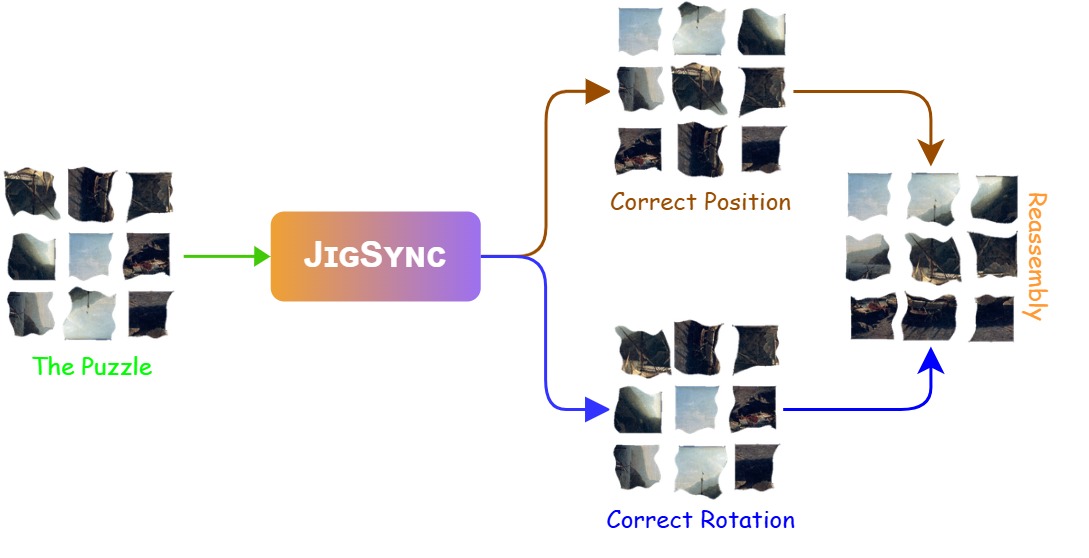}
  \caption{\textbf{\method} is a gauge-resolved reassembly framework for
  degraded, shuffled and arbitrarily rotated pieces. It synchronizes
  pairwise pose relations, anchors the otherwise unobservable global
  orientation on content cues, and recovers each piece's position and
  rotation.}
  \label{fig:teaser}
\end{figure}

\paragraph{Square jigsaw solving.} Jigsaw reassembly was formulated
computationally in the 1960s~\cite{freeman1964} and later shown
NP-complete~\cite{demaine2007}. Classical approaches paired a hand-crafted
pairwise compatibility with global optimizer, using strategies such as greedy
placement~\cite{pomeranz2011}, genetic search~\cite{sholomon2013}, linear
programming~\cite{yu2015}, and relaxation
labeling~\cite{khoroshiltseva2021,vardi2023}. Early learning-based approaches replaced hand-crafted compatibility measures with learned ones~\cite{sholomon2016,paumard2020,jigsawgan}. Later work explored
generative and diffusion
formulations~\cite{jpdvt,giuliari2024,diffassemble,talon2022,talon2025}
and reinforcement learning~\cite{sd2rl,pdnga,erlmpp,ceari}. Transformer
solvers built on ViT~\cite{vit} predict positions directly~\cite{heck2025} or regress fragment
coordinates~\cite{fcvit}, while multimodal solvers add language
supervision~\cite{vlhsa} or recast reassembly over discrete
tokens~\cite{puzlm}. Vision-language models have been evaluated as
solvers under constrained settings~\cite{lyu2025,jigsawr1}, and jigsaw
solving is separately used as a self-supervised pretext
task~\cite{noroozi2016,wei2019,jigsawvit,ren2023}. Unknown piece
orientation was also studied in the classical setting by
Gallagher~\cite{gallagher2012}, who considered square jigsaws with
unknown piece orientations, including a tree-based procedure for
component merging and a pairwise Markov Random Field (MRF) formulation for
known locations and unknown orientations. Gallagher~\cite{gallagher2012}
also introduced Mahalanobis Gradient Compatibility (MGC), which scores
candidate matches from the smoothness of gradient distributions across
shared boundaries. Our setting differs in that the fragments are
eroded and irregularly shaped, so the boundary assumptions underlying
such compatibility measures no longer hold. \Cref{sec:cues} quantifies
a compatibility term of this family and finds it falls below chance as
fragment shape becomes irregular.

% \paragraph{Unknown piece orientation.} Orientation is not a new question.
% Gallagher~\cite{gallagher2012} introduced square jigsaw puzzles in which
% each piece's orientation is unknown, solved them by a tree-based
% procedure that greedily merges components under the geometric
% constraints, formulated the case of known location and unknown
% orientation as a pairwise Markov random field (MRF) over per-piece
% orientation, and proposed Mahalanobis Gradient Compatibility (MGC), which
% scores a candidate match by the smoothness of the gradient distribution
% across the shared boundary. The gap we identify is therefore in the
% \emph{benchmarks} rather than in the literature: those on which current
% solvers are compared, including the one we report on, provide fragments
% upright, leaving the orientation degree of freedom neither exercised nor
% scored. The fragments also differ. MGC and related measures read a band
% of pixels either side of a shared boundary, presuming the two boundaries
% abut and remain intact, whereas the fragments we treat are eroded and
% irregularly shaped, so adjacent boundaries do neither.
% \Cref{sec:cues} quantifies a compatibility term of this family and finds
% it falls below chance as fragment shape becomes irregular.

\paragraph{Eroded and irregular fragments.} A parallel line drops the
square-fragment assumption: adversarial inpainting across eroded
gaps~\cite{bridger2020}, twin embeddings~\cite{ten}, crossing-cut
polygonal puzzles~\cite{harel2024}, archaeological
reconstruction~\cite{derech2021}, pairwise alignment for arbitrary
fragments~\cite{shahar2025}, and the scanned RePAIR
benchmark~\cite{repair}. Closest to us is GAP~\cite{puzzleflow}, which
generates irregular eroded fragments from a distribution learned on
RePAIR and pairs them with a flow-matching solver. GAP is our primary
comparator and likewise provides all fragments upright.

\paragraph{Group synchronization.} The natural machinery for the
generalized problem is
synchronization~\cite{singer2011,bandeira2017,cucuringu2012}: estimate
noisy \emph{relative} poses on pairs, then solve for globally consistent
absolute poses. The decomposition is attractive because relative pose is
locally measurable while absolute pose is not. That relative measurements
determine absolute ones only up to a global constant is standard; what
appears not to be stated, and is silently relied on in practice, is the
consequence for the downstream translation estimate, which we give in
\cref{sec:gauge}. Robust variants of that estimate, notably translation
synchronization by truncated least squares~\cite{huang2017}, are
discussed in \cref{sec:consequences}.

% \paragraph{Our finding.} A synchronization pipeline built on relative
% measurements determines every fragment's pose only up to one free global
% rotation, and \emph{cannot recover that rotation by any residual
% criterion}. Selecting among the four candidate gauges by comparing residuals is not a
% noisy decision but no decision at all: all four agree to 10--12 decimal
% places on real and synthetic data alike, and a single inverted tie moves
% a downstream evaluation between $100\%$ and $11\%$ on individual
% puzzles. The theorem also identifies what
% resolves the issue, since it constrains anything derived from the
% pairwise graph and nothing else: an orientation estimated from the
% content of a single fragment lies outside its scope. Such an anchor
% raises the GAP-3 selection ceiling from \ours{$16.61\%$} to
% \ours{$36.72\%$} and the GAP-5 ceiling from \ours{$5.38\%$} to
% \ours{$10.66\%$}, the largest single effect in our experimental program;
% six other inference-level interventions each move it by less than $1.2$
% points.

\paragraph{Our finding.} A synchronization pipeline based on relative
measurements determines every fragment's pose only up to one free global
rotation and cannot recover that rotation by any residual criterion.
All four candidate gauges agree to 10--12 decimal places, while a single
inverted tie can change downstream evaluation substantially. The theorem
also identifies the resolution: an orientation estimated from the content
of a single fragment lies outside the pairwise graph's scope and can fix
the gauge. This anchor raises the GAP-3 selection ceiling from
\ours{$16.61\%$} to \ours{$36.72\%$} and the GAP-5 ceiling from
\ours{$5.38\%$} to \ours{$10.66\%$}, the largest effect in our
experimental program.

\paragraph{Contributions.}
\begin{enumerate}[leftmargin=1.3em,itemsep=0pt,topsep=2pt]
  \item A gauge-unobservability theorem for least-squares translation
    estimation, showing that the fitting penalty cannot between distinguish
    global orientations, together with the resulting content-anchored resolution
    the theorem implies (\cref{sec:gauge}).
  \item An orbit-centroid degeneracy result explaining constant-predictor
    collapse in per-fragment pose regression, with a min-over-gauge
    objective that resolves it (\cref{sec:orbit}).
  \item \method, a solver for $\Z_4 \wr \SN$ combining a dense tiered
    relative-pose vocabulary, bearing-only anisotropic edge weighting,
    and gauge-anchored $\rho$-conditioned decoding (\cref{sec:method}).
  \item \protoname, a degradation protocol that recovers the classical
    square-piece setting exactly at zero degradation and interpolates to
    destroyed-geometry, arbitrarily-oriented photographs
    (\cref{sec:protocol}).
\end{enumerate}

%------------------------------------------------------------------------
\section{Problem Setting}
\label{sec:problem}

\subsection{Placement and orientation together}
\label{sec:posedef}

A puzzle is a set of $N$ fragments $X = \{x_1,\dots,x_N\}$ cut from an
unknown source image and provided to the solver in arbitrary order, each
having additionally been turned by a multiple of $90^\circ$. Solving it
means producing a \textbf{placement} $\pi$, assigning each fragment to
one of the $N$ grid slots with no two sharing a slot, and an
\textbf{orientation} $\rho_i \in \{0^\circ,90^\circ,180^\circ,270^\circ\}$
per fragment, chosen independently. We write the pair as $g^{*} = (\rho,
\pi)$. Placements form the symmetric group $\SN$, the four orientations
form the cyclic group $\Z_4$, and pairs of the two compose in the manner
of a \emph{wreath product} $\Z_4 \wr \SN$, the formal name for ``$N$
independently rotatable objects distributed among $N$ labelled slots'',
of size
\begin{equation}
  |\Z_4 \wr \SN| = 4^{N} \cdot N!
  \label{eq:wreath}
\end{equation}
against $N!$ for placement alone. Grid size is not supplied at test time.
\Cref{fig:gallery} shows one instance as the solver receives it.

\begin{figure}[t]
  \centering
  \includegraphics[width=\linewidth]{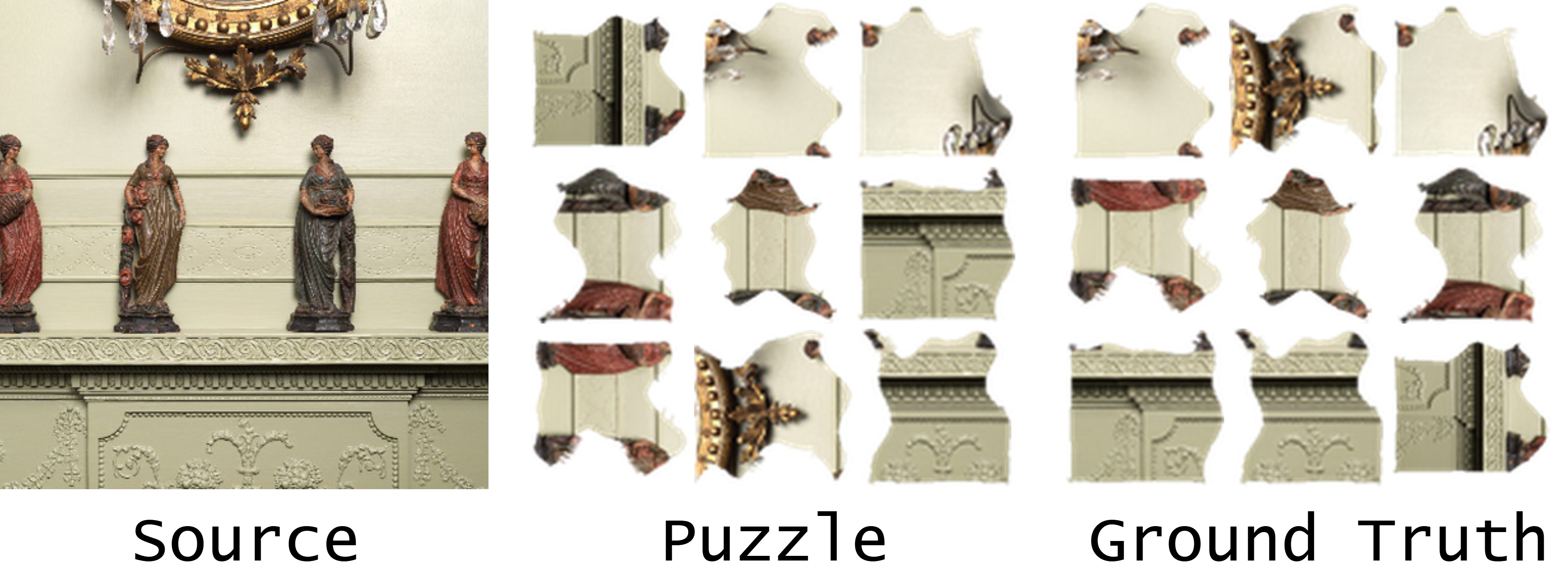}
  \caption{A \protoname\ instance. The source photograph (left) is cut into
  irregular eroded fragments, provided shuffled \emph{and} independently
  rotated (centre); ground-truth slots are at right, rotation still
  unresolved. GAP and its predecessors exercise the shuffle but not the
  rotation.}
  \label{fig:gallery}
\end{figure}

\subsection{Relative measurement and the global gauge}
\label{sec:syncintro}

Enumerating $4^N N!$ candidates is infeasible, and predicting a
fragment's slot directly is little better, since a fragment carries
little information about its absolute position. What a \emph{pair} of
fragments does carry information about is their relationship: whether
they were neighbours, on which side, and how much one is turned relative
to the other. For each ordered pair $(i,j)$ we therefore estimate a
\textbf{relative rotation} $\hat\rho_{ij} \in \Z_4$ and a \textbf{relative
displacement} $\hat\delta_{ij} \in \R^2$. Turning a collection of noisy
relative estimates into one globally consistent set of absolute values is
the problem known as \emph{synchronization}, solved rotations first,
because displacements can only be compared once all fragments share a
common frame.

That ordering is what creates the difficulty. Rotation synchronization
returns each fragment's absolute rotation only up to one shared unknown:
if $\hat\theta_1,\dots,\hat\theta_N$ is a solution, so is
$\hat\theta_i + c$ for any fixed $c \in \Z_4$, since every relative
rotation $\hat\theta_j - \hat\theta_i$ is unchanged. The constant $c$ is
what we call the \textbf{global gauge}, after the physics usage in which
a gauge is a degree of freedom that changes the description of a system
without changing anything measurable about it. It has four values, one of
which is right, and it cannot be left unresolved: the relative
displacements must be expressed in the recovered frame before they are
combined, so a wrong $c$ rotates every displacement in the puzzle and the
reconstruction comes out turned with it. Some choice must be made.
\Cref{sec:gauge} shows that the standard way of making it carries no
information whatsoever.

%------------------------------------------------------------------------
\section{Gauge Unobservability}
\label{sec:gauge}

\subsection{The position fit}

Once rotations are fixed, absolute positions $\{x_i \in \R^2\}$ are
recovered as the layout disagreeing least with the measured
displacements, each pair penalizing the mismatch between where the layout
puts $j$ relative to $i$ and where the measurement said it should be:
\begin{equation}
  x^* = \argmin_{x} \sum_{(i,j)\in E}
    (x_j - x_i - \hat\delta_{ij})^\top W_{ij}
    (x_j - x_i - \hat\delta_{ij}).
  \label{eq:transsync}
\end{equation}
Here $E$ is the set of measured pairs and $W_{ij}$ a $2 \times 2$
positive-semidefinite matrix expressing how far a pair is trusted and in
which direction; an anisotropic $W_{ij}$ lets it be confident about
direction while uncertain about distance, which \cref{sec:bearing}
exploits. With every $W_{ij}$ a multiple of the identity this is ordinary
graph-Laplacian least squares, one sparse solve after fixing a single
fragment's position to remove the residual freedom to translate the
layout. Since rotation synchronization left $c$ undetermined, the natural
procedure is to run \cref{eq:transsync} once per candidate $c$, with
every $\hat\delta_{ij}$ rotated accordingly, and keep the smallest
penalty.

\subsection{The theorem}

\begin{theorem}
\label{thm:gauge}
Let $G=(V,E)$ carry relative-position measurements $\hat\delta_{ij} \in
\R^2$ with positive-semidefinite weights $W_{ij}$, and let $J$ be the
objective of \cref{eq:transsync}. Let $R \in O(2)$ be any orthogonal
transformation, in particular a rotation by an arbitrary angle, and
define the uniformly rotated measurement set $\hat\delta'_{ij} =
R\hat\delta_{ij}$, $W'_{ij} = RW_{ij}R^\top$, with objective $J'$. Then
\begin{equation}
  \min_x J'(x) = \min_x J(x),
  \label{eq:gaugeinvariance}
\end{equation}
and the minimizer of $J'$ is $Rx^*$, where $x^*$ minimizes $J$.
Moreover, the equality holds edgewise: each pair's penalty is
individually unchanged.
\end{theorem}

\begin{proof}
For any $x$, write $x = Rx'$ with $x' = R^\top x$. Per edge,
$x_j - x_i - \hat\delta'_{ij} = R(x'_j - x'_i - \hat\delta_{ij}) =: Rv$,
so that edge's penalty is
$(Rv)^\top (RW_{ij}R^\top)(Rv) = v^\top R^\top R\, W_{ij}\, R^\top R\, v
= v^\top W_{ij} v$,
by orthogonality applied twice, which establishes the edgewise claim.
Summing over edges gives $J'(Rx') = J(x')$ for every $x'$, and hence
$\min_x J'(x) = \min_{x'} J(x')$.
\end{proof}

\subsection{Consequences}
\label{sec:consequences}

\begin{corollary}
\label{cor:tiebreak}
Any procedure that selects among the four candidate global gauges by
comparing achieved penalties ties \emph{exactly} across all four.
Whichever candidate such an implementation reports is determined by
floating-point rounding, not by evidence.
\end{corollary}

The four quantities compared are equal as real numbers rather than
merely close, verified to ten or twelve decimal places on real and
synthetic inputs alike. The proof uses only $R^\top R = I$, so it is indifferent to how
many pairs were measured, how they were weighted, and how noisy they
are.

\noindent\textbf{Robust estimators inherit the invariance.} One might
expect a robust variant of \cref{eq:transsync} to behave differently.
Translation synchronization by truncated least squares~\cite{huang2017}
discards or downweights an edge by the magnitude of its residual, and the
theorem holds edgewise, so the same edges are truncated at the same
values under every global rotation and the attained minimum is again
identical. The argument extends to any $M$-estimator whose loss reads the
residual only through $v^\top W_{ij} v$. Robustification addresses
outlying measurements; it does not make the gauge observable.

\noindent\textbf{A remark from gauge theory.} The result follows the
principle that physical observables must be gauge-invariant. The theorem says precisely that the fitting penalty
\emph{is} such a quantity, and a gauge-invariant quantity cannot
distinguish between gauges. Global orientation is therefore not an
observable of the pairwise measurement system at all but a gauge freedom,
which can only be pinned down against a reference external to that
system. The single-fragment anchor of \cref{sec:orbit} is that reference,
and gauge fixing is the right name for what it does. This also accounts for two earlier observations: scoring all four
candidates rather than the first changed nothing downstream, and a family
of edge-weighting and consistency-filtering refinements measured as null
end to end. Both were read through a decision layer carrying no
information.

\subsection{Why the obvious estimator collapses}
\label{sec:orbit}

The anchor we need predicts a fragment's absolute orientation, and
ideally its grid cell, from that fragment alone. We call it the
\textbf{unary head}, since it consumes one fragment where the pairwise
model consumes two. Trained in the obvious way it fails, for reasons independent of
optimization. Consider its target when each puzzle is presented in a random
gauge: rotating the grid by $90^\circ$ sends the top-left cell to the
top-right, and again to the bottom-right, so any cell has four possible
targets. This set of four is the cell's \textbf{orbit} under the rotation
group.

\begin{definition}
For a fragment at ground-truth cell $(r,c)$ in an $n \times n$ grid and
gauge $k \in \Z_4$, let $\mathrm{rot}_k(r,c)$ be the cell occupied by
$(r,c)$ after the grid is rotated by $k \cdot 90^\circ$ about its centre,
in coordinates normalized to $[0,1]^2$.
\end{definition}

\begin{proposition}
\label{prop:orbit}
For every grid cell $(r,c)$,
$\ \tfrac14\sum_{k=0}^{3}\mathrm{rot}_k(r,c) = (\tfrac12,\tfrac12)$.
\end{proposition}

\begin{proof}[Proof sketch]
Let $R$ be the $90^\circ$ rotation about the grid centre acting on
centred coordinates $v = (r,c) - (\tfrac12,\tfrac12)$. In the eigenbasis
of $R$, $\sum_{k=0}^3 R^k$ acts as $1 + \lambda + \lambda^2 + \lambda^3 =
(1-\lambda^4)/(1-\lambda) = 0$ for every fourth root of unity $\lambda
\neq 1$, and the sole $\lambda = 1$ direction is the grid centre itself,
against which $v$ was already measured. Hence $\sum_k R^k v = 0$. Full
proof in the supplement.
\end{proof}

Every cell's four targets average to the same point, the grid centre.
Since squared-error loss is minimized by predicting the mean of the
target, the consequence is immediate.

\begin{corollary}
\label{cor:collapse}
If the gauge $k$ is drawn uniformly per training example and cannot be
inferred from a single fragment, then under squared-error or Gaussian
likelihood loss the best possible prediction, \emph{for every fragment
regardless of what it depicts}, is the grid centre.
\end{corollary}

We observed exactly this: cell accuracy at chance for every epoch, and
the coordinate loss settling at $0.3335$ against the value $1/3$ that
\cref{cor:collapse} predicts in closed form.

\noindent\textbf{The min-over-gauge objective.} The collapse follows from
the loss averaging over gauges, so we take the best of the four instead,
letting gradient flow only through the winner,
\begin{equation}
  \mathcal{L} = \min_{k \in \Z_4}\Bigl[
    \lambda_{c}\,\mathrm{NLL}\bigl(\hat{xy},\, \mathrm{rot}_k(r,c)\bigr)
    + \mathrm{CE}\bigl(\hat\rho,\, (k - r_i) \bmod 4\bigr)\Bigr],
  \label{eq:minovergauge}
\end{equation}
where $\hat{xy}$ and $\hat\rho$ are the head's cell and rotation
predictions, $r_i$ the rotation applied to fragment $i$, and $\lambda_c$
balances the terms. A constant predictor cannot sit close to all four orbit points at once,
so the objective rewards content-dependent predictions self-consistent
under \emph{some} gauge, deferring the choice to inference. On upright
GAP the head reaches \ours{$93.6\%$} rotation and \ours{$33.9\%$} cell
accuracy, three times chance, with \ours{$94.1\%$} and \ours{$11.5\%$} on
GAP-5. That rotation figure, far above the $44$--$52\%$ the pairwise model
reaches on \emph{relative} rotation, is what makes the head usable as an
anchor.

%------------------------------------------------------------------------
\section{The \protoname\ Protocol}
\label{sec:protocol}

Existing benchmarks vary damage as one aggregate quantity, so they cannot
isolate which form of damage defeats which cue. We therefore constructed
a generator with an independently controlled axis per form of
degradation, in which setting every axis to zero reproduces the classical
square-piece puzzle exactly: shape irregularity $A$, boundary erosion
$\varepsilon_{geo}$, photometric corruption $\varepsilon_{photo}$, and
grid size $n$. Full formulas and unit tests are in the supplement.

\noindent\textbf{Conjugate fracture tessellation.} An $n \times n$ grid
has $2n(n-1)$ internal edge segments, each shared by two cells. If
neighbouring cells were cut along even slightly different curves their
fragments could never mate, so both must use the \emph{same} curve. For
each internal edge from $a$ to $b$ we displace the straight chord
sideways by a sum of $K$ sine harmonics,
\begin{equation}
  a_k \sim \mathcal{N}\!\left(0, \tfrac{A}{k}\right), \quad
  d(t) = \sum_{k=1}^{K} a_k \sin(k\pi t), \quad t \in [0,1],
  \label{eq:fracture}
\end{equation}
placing the curve at $a + t(b-a) + d(t)\hat n$ for chord normal $\hat n$.
Since $\sin(0) = \sin(k\pi) = 0$, every curve is pinned at both endpoints
for any $A$, and $A = 0$ gives straight edges. Each cell walks its four
surrounding curves in a consistent order, so a shared edge is the same
sampled point sequence for both neighbours. We sweep $A \in
\{0,8,16,24,32\}$.

\noindent\textbf{Erosion and photometric damage.} Erosion is applied per
fragment \emph{after} cutting, so the two sides of a former seam no
longer match: a boundary point at normalized arc-length $s$ is pushed
inward along its local normal by $e(s) = \varepsilon_{geo} P (0.5 +
\eta(s))$, with $P$ the fragment resolution and $\eta$ periodic Perlin
noise~\cite{perlin1985} at three octaves, plus Poisson-distributed disk
bites at $\varepsilon_{geo} \ge 0.10$. Spatial variation is deliberate: uniform morphological erosion would
leave the result directly predictable from the original boundary. Photometric
damage adds global gain, sepia desaturation, noise and a JPEG
round-trip, plus a boundary-weighted fade $w(x) =
\exp(-d_\partial(x)/0.15P)$ attenuating the pixels nearest the fragment
edge, exactly those a boundary-matching cost reads.

\noindent\textbf{Rotation and presentation order.} Each fragment is turned by an independent $r_i \sim
\mathrm{Uniform}\{0,90,180,270\}$, applied by array rotation rather than
resampling: an arbitrary angle leaves interpolation blur along fragment
edges from which a solver could recover orientation without reading
content at all (\cref{sec:limitations} returns to this). A uniform
permutation then fixes the presentation order. \emph{This is the axis GAP and its predecessors leave at zero}, and the
reason the gauge problem arises here and not there. Fragments are cut from $12{,}000$ CC0 Metropolitan Museum of Art
images~\cite{metopenaccess,ypsilantis2021}, drawn from the same
catalogues GAP is built from and verified free of cross-split overlap,
yielding more than $100$ thousand puzzles at $n \in \{3,4,5,6,8,10\}$.

%------------------------------------------------------------------------
\section{\method}
\label{sec:method}

\method\ combines the preceding two sections (\cref{fig:jigsync}): a
pairwise model measures relative poses, rotation synchronization
reconciles them up to the gauge, the unary head supplies the gauge, and
position fitting produces the layout.

\begin{figure*}[t]
  \centering
  \includegraphics[width=\textwidth]{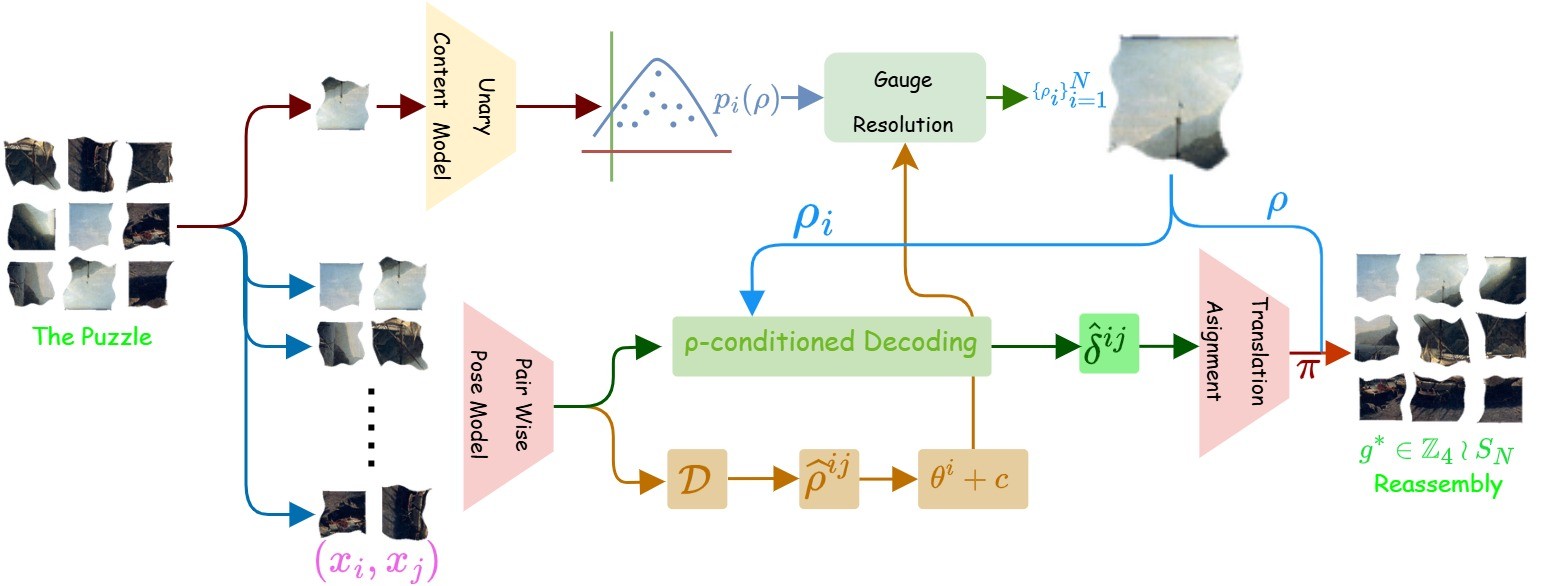}
  \caption{\textbf{\method} recovers the global pose of shuffled and
  rotated puzzle pieces. The Pairwise Pose Model estimates relative
  rotations and positions from degraded pieces $\{x_i\}_{i=1}^{N}$.
  Synchronizing the relative rotations $\hat{\rho}_{ij}$ leaves the free
  offset $c$, which no fitting penalty can determine (\thmref); it is
  resolved instead by the unary orientation probabilities $p_i(\rho)$,
  computed one fragment at a time and hence independent of the pairwise
  measurements. The resulting absolute rotations $\{\rho_i\}_{i=1}^{N}$
  condition pairwise decoding to give relative positions
  $\hat{\delta}_{ij}$, followed by translation synchronization and
  assignment into the final pose $g^*=(\rho,\pi)$.}
  \label{fig:jigsync}
\end{figure*}

\subsection{A dense, tiered pose vocabulary}
\label{sec:vocab}

The pairwise model predicts, per ordered pair, one class encoding both
relative displacement and relative rotation. One class per possible
$(\Delta\mathrm{row}, \Delta\mathrm{col})$ is poor at long range: exact
distance between distant fragments is not recoverable from appearance and
demanding it injects label noise, whereas direction often is recoverable
and relative rotation is recoverable at any distance. We therefore let
resolution degrade with distance, in three \textbf{tiers} indexed by the
Chebyshev distance $r = \max(|\Delta\mathrm{row}|,|\Delta\mathrm{col}|)$:
\begin{equation}
  \mathrm{tier}(r) = \begin{cases}
    \text{near}, & r \in \{1,2\}:\ \text{exact } (\Delta r, \Delta c)\\
    \text{mid},  & r \in \{3,4\}:\ \text{8-way direction}\\
    \text{far},  & r \ge 5:\ \text{8-way} \times 2\ \text{bands}
  \end{cases}
  \label{eq:tiers}
\end{equation}
giving $24 + 8 + 16 = 48$ position buckets times $4$ relative rotations,
$192$ classes. Everything is expressed in fragment $i$'s \emph{presented}
frame, and the buckets are laid out so that turning that frame by
$k \cdot 90^\circ$ shifts the direction bucket by exactly $2k \bmod 8$,
which is what makes rotation and position separable at decode time.
Supervising only grid-adjacent pairs, the standard practice, leaves $17$
live classes on a $3\times3$ grid and discards most of the available
signal.

\subsection{Rotation synchronization}
\label{sec:rotsync}

Given noisy relative rotations $\hat\rho_{ij}$ with confidences $w_{ij}$,
we represent each fragment's unknown rotation as a point on the unit
circle and assemble the Hermitian matrix
\begin{equation}
  H_{ij} = w_{ij}\, e^{i(\theta_i - \theta_j)}
         \approx w_{ij}\, e^{-i\frac{\pi}{2}\hat\rho_{ij}},
  \qquad H_{ji} = \overline{H_{ij}},
  \label{eq:connmatrix}
\end{equation}
whose leading eigenvector gives $\hat\theta_i = \arg(v_i)$, rounded to
the nearest multiple of $90^\circ$. Being a global least-squares fit on the circle rather than propagation
along a chain of edges, it averages a single erroneous measurement
against the rest instead of propagating it, and is exact when
measurements are mutually consistent. A cycle-consistency filter halves
the weight of all three edges of any triangle whose rotations do not sum
to zero. The recovered $\hat\theta_i$ are correct only up to $c$.

\subsection{Bearing-only position fitting}
\label{sec:bearing}

Mid- and far-tier edges report direction but not distance, so they cannot
supply a point estimate of $\hat\delta_{ij}$ without fabricating one.
They instead enter \cref{eq:transsync} through an anisotropic weight that
consumes the measurement as a constraint on direction alone,
\begin{equation}
  W_{ij} = \lambda_{\perp}\bigl(I - \hat d \hat d^\top\bigr)
         + \lambda_{r}\, \hat d \hat d^\top,
  \qquad \hat d = \hat\delta_{ij}/\|\hat\delta_{ij}\|,
  \label{eq:bearing}
\end{equation}
penalizing deviation \emph{across} the measured bearing with weight
$\lambda_\perp$ and \emph{along} it with a much smaller $\lambda_r$;
setting $\lambda_r = \lambda_\perp$ recovers the isotropic case exactly,
which serves as a correctness check. Continuous positions are mapped to
grid slots by Hungarian assignment~\cite{kuhn1955}, enforcing the
one-fragment-per-slot constraint that independent rounding would violate.

\subsection{Gauge anchoring and $\rho$-conditioned decoding}
\label{sec:anchor}

The unary head is used twice. As the \textbf{gauge anchor}, we pick the
$c$ best agreeing with its absolute rotations rather than comparing the
four fitting penalties, which \cref{cor:tiebreak} shows is vacuous. For
\textbf{$\rho$-conditioned decoding}, once $c$ is fixed every pair's
relative rotation is known, so the pairwise argmax is restricted to the
$48$ consistent classes rather than all $192$. This turns the head's
strong signal into a hard constraint on the pairwise model's weaker one;
the asymmetry justifying it is roughly $94\%$ against $44$--$52\%$.

\subsection{Density of the measurement graph}
\label{sec:density}

Grid adjacency gives each fragment approximately four neighbours
irrespective of puzzle size; an all-pairs graph gives it $N-1$.
Synchronization theory favours the latter, since recovery requires
roughly $q \gtrsim 1/\sqrt{d}$ for per-edge accuracy $q$ and degree $d$,
so a fixed degree must fail as $N$ grows, and a synthetic sweep
reproduces this. Measured accuracies give a different result, because our
long-range errors are not uniformly distributed: of the incorrect
Chebyshev-2 predictions, $99.88\%$ fall in the near tier and $65.9\%$
report the pair as adjacent. The errors are systematically biased toward
shorter distances. Repeating the sweep with a corruption model matched to
this bias, at $150$ seeds per point over a $17$-point grid
(\cref{fig:qfar}), places the break-even far-tier accuracy at
\begin{equation}
  q_{far}^*(n) \approx
    \begin{cases} 0.21, & n = 3\\ 0.07, & n = 5,\end{cases}
  \label{eq:qfarstar}
\end{equation}
above which the denser graph is advantageous and below which it reduces
accuracy. Our measured $q_{far} \approx 0.132$ lies \emph{below} the
$n{=}3$ threshold, where the sweep predicts a $3.9$-point cost, matching
the $-3.4$ points measured end to end, and \emph{above} the $n{=}5$
threshold, where it predicts a $1.6$-point gain. The larger grid, on which a dense graph appears least economical, is
where it should already be advantageous. This is a prediction:
the retrain testing it has not been run, \cref{fig:qfar} locates each
threshold only to $\pm0.05$--$0.10$, and the $n{=}3$ curve reverses
slightly near $q_{far} = 0.28$ and $0.42$.

% \begin{figure}[t]
%   \centering
%   \includegraphics[width=\linewidth]{img/fig21_qfar_crossover_sweep.png}
%   \caption{\textbf{Densification crossover.} All-pairs AA against per-edge
%   far-tier accuracy $q_{\mathrm{far}}$, 150 seeds per point, $n{=}3$
%   (left) and $n{=}5$ (right). Dashed: the grid-adjacency reference, flat
%   in $q_{\mathrm{far}}$ by construction. Shaded: the crossing, widened to
%   the $\pm0.05$--$0.10$ resolution of $q^{*}_{\mathrm{far}}$. Dotted: our
%   measured $q_{\mathrm{far}}{=}0.132$. Note the differing vertical
%   scales.}
%   \label{fig:qfar}
% \end{figure}

%------------------------------------------------------------------------
\section{Experiments}
\label{sec:experiments}

\subsection{Experimental setup}

\noindent\textbf{Datasets.} We evaluate on GAP-3 and
GAP-5~\cite{puzzleflow} ($9$ and $25$ irregular eroded fragments,
$128 \times 128$ RGBA with the alpha channel carrying the fragment mask,
released $14{,}000/3{,}000/3{,}000$ splits) and on \protoname, which
unlike GAP varies fragment orientation.

\noindent\textbf{Metrics.} Following prior
work~\cite{puzzleflow,sd2rl,paumard2020} we report \textbf{Perfect
Accuracy (PA)}, the percentage of puzzles solved completely;
\textbf{Absolute Accuracy (AA)}, the percentage of fragments in the
correct cell; and \textbf{Spatial Relationship Accuracy (SRA)}, the
percentage of ground-truth neighbour pairs retaining the same relative
configuration,
\begin{equation}
    \text{SRA} = \frac{1}{M}\sum_{i=1}^{M}
    \frac{|\{(u,v,d) \in \mathcal{N} : \text{rel}(\mathbf{p}_i,u,v) = d\}|}{|\mathcal{N}|},
    \label{eq:sra}
\end{equation}
where $\mathcal{N}$ collects neighbour pairs with directions $d \in
\{\text{left},\text{right},\text{up},\text{down}\}$, following
\cite{sd2rl,gur2017}: a pair is credited only if its ground-truth
direction is preserved. High SRA with moderate AA indicates local
structure recovered but globally misplaced. For the single-fragment head
we also report \emph{rotation} and \emph{cell accuracy}, with chance
$1/4$ and $1/n^2$.

\noindent\textbf{Two evaluation variants.} Variant \textbf{A} supplies
the true grid adjacency, leaving only the poses on those \emph{oracle
edges} to be estimated. It isolates pose measurement from edge selection
and, bounding what better selection could achieve, gives a
\emph{selection ceiling}. Variant \textbf{C} uses the model's own dense
all-pairs graph and is the deployable configuration. Neither is supplied
a pose.

\noindent\textbf{Split protocol.} Comparison against published methods
uses the full $3{,}000$-puzzle test splits, matching the protocol under
which those numbers were produced. Ablations use a fixed $2000$-puzzle
development subset and report relative differences only; absolute values
there are not comparable to full-split numbers.

\begin{figure}[t]
  \centering
  \includegraphics[width=\linewidth]{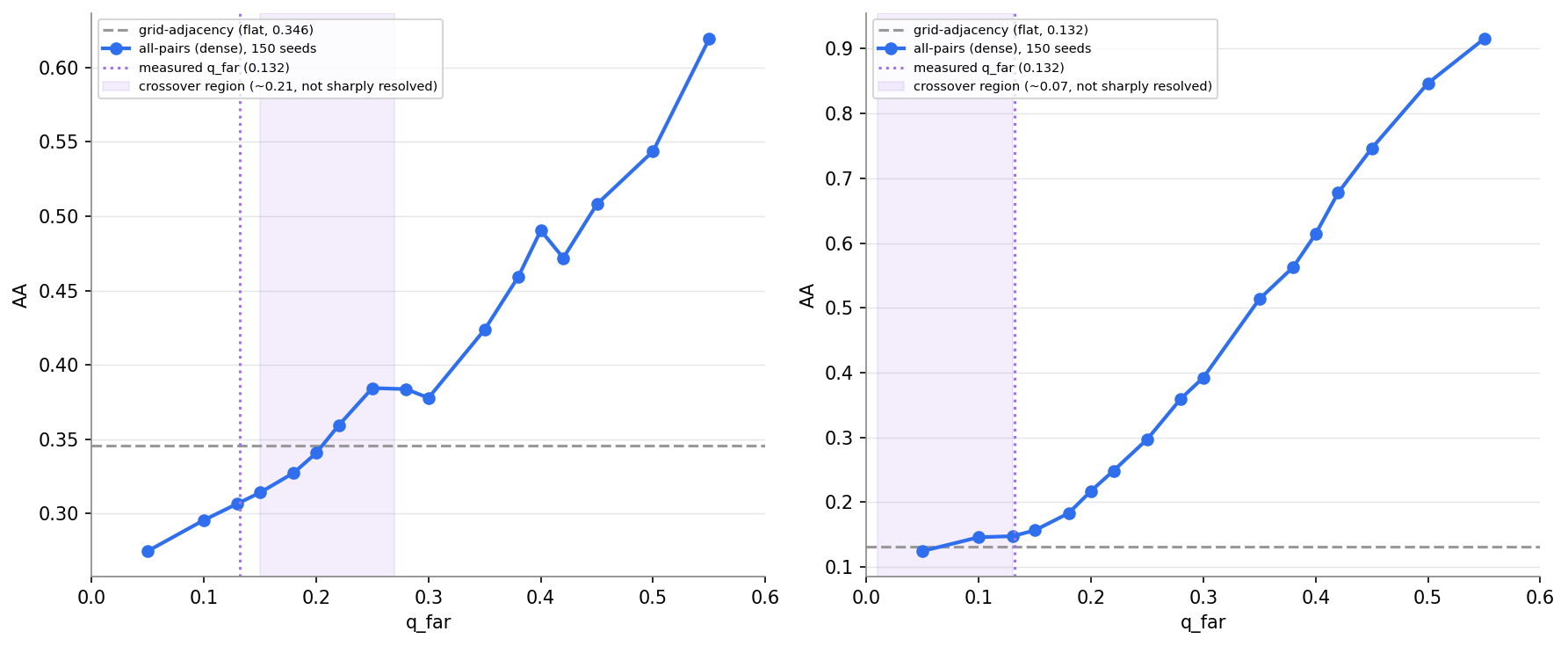}
  \caption{\textbf{Densification crossover.} All-pairs AA against per-edge
  far-tier accuracy $q_{\mathrm{far}}$, 150 seeds per point, $n{=}3$
  (left) and $n{=}5$ (right). Dashed: the grid-adjacency reference, flat
  in $q_{\mathrm{far}}$ by construction. Shaded: the crossing, widened to
  the $\pm0.05$--$0.10$ resolution of $q^{*}_{\mathrm{far}}$. Dotted: our
  measured $q_{\mathrm{far}}{=}0.132$. Note the differing vertical
  scales.}
  \label{fig:qfar}
\end{figure}

\subsection{The gauge fix}
\label{sec:expgauge}

\Cref{tab:gaugefix} isolates the central empirical claim. The first row
is the standard procedure \cref{cor:tiebreak} shows to be vacuous, the
second replaces it with the unary anchor, and all six values improve,
both absolute accuracies more than doubling. The third supplies the
correct gauge as a diagnostic bound, yielding further AA and SRA but no
further PA, already saturated at the level oracle edges permit.

\begin{table}[t]
  \centering\small
  \setlength{\tabcolsep}{3.5pt}
  \begin{tabular}{@{}lrrrrrr@{}}
    \toprule
    & \multicolumn{3}{c}{GAP-3} & \multicolumn{3}{c}{GAP-5} \\
    \cmidrule(lr){2-4}\cmidrule(lr){5-7}
    Gauge resolution & PA & AA & SRA & PA & AA & SRA \\
    \midrule
    Fitting penalty (vacuous) & \ours{2.50} & \ours{16.61} & \ours{17.46} & \ours{0.00} & \ours{5.38} & \ours{6.21} \\
    Unary anchor (deployable) & \oursb{11.00} & \oursb{36.72} & \oursb{31.54} & \ours{0.00} & \oursb{10.66} & \oursb{8.88} \\
    \quad + true gauge (ceiling) & \ours{11.00} & \ours{41.94} & \ours{31.87} & \ours{0.00} & \ours{12.84} & \ours{9.03} \\
    \bottomrule
  \end{tabular}
  \caption{\textbf{Effect of the gauge fix.} Variant A (oracle edges,
  model poses), $2000$-puzzle development subset.}
  \label{tab:gaugefix}
\end{table}

A six-way ablation on the anchored baseline adds one component at a time:
masking classes impossible for the grid size, the head's \emph{position}
estimate as a prior, the anisotropic bearing weights, down-weighting
edges failing cycle-consistency, and iterating fit and assignment to
convergence. These move GAP-3 AA by $+0.11$, $-1.11$, $0.00$, $-0.39$ and
$-0.66$, and GAP-5 by $-0.04$ for bearing weights; against the gauge
fix's roughly $20$ points, all are null.

\subsection{From components to a deployable pipeline}

\Cref{tab:progression} traces the pipeline without retraining.
Substituting dense for oracle edges reduces accuracy, which
\cref{sec:density} attributes to far-tier accuracy lying below
$q_{far}^*$ at $n{=}3$; tier weighting has no effect. $\rho$-conditioned
decoding recovers that loss and more, giving the best AA and SRA of any
configuration without oracle edges. Exceeding variants with privileged
edge information indicates that pose \emph{measurement}, not edge
\emph{selection}, is the remaining bottleneck. It does not lead on PA.

\begin{table}[t]
  \centering\small
  \setlength{\tabcolsep}{3pt}
  \begin{tabular}{@{}lrrrrrr@{}}
    \toprule
    & \multicolumn{3}{c}{GAP-3} & \multicolumn{3}{c}{GAP-5} \\
    \cmidrule(lr){2-4}\cmidrule(lr){5-7}
    Stage & PA & AA & SRA & PA & AA & SRA \\
    \midrule
    A: oracle edges, anchored        & \ours{11.00} & \ours{36.72} & \ours{31.54} & \ours{0.00} & \ours{10.66} & \ours{8.88} \\
    C: dense edges, uniform weight   & \ours{4.00}  & \ours{33.28} & \ours{24.71} & \ours{0.00} & \ours{8.90}  & \ours{5.90} \\
    \quad + tier weighting (null)    & \ours{4.00}  & \ours{33.22} & \ours{24.83} & \ours{0.00} & \ours{8.92}  & \ours{5.88} \\
    \quad + $\rho$-conditioned decode& \ours{8.50}  & \oursb{41.33} & \oursb{33.21} & \ours{0.00} & \oursb{12.26} & \ours{8.80} \\
    A: stack-matched, oracle edges   & \oursb{17.00} & \ours{43.39} & \ours{40.00} & \ours{0.00} & \ours{13.80} & \oursb{15.90} \\
    \bottomrule
  \end{tabular}
  \caption{\textbf{Pipeline progression}, no retraining, $2000$-puzzle
  development subset. The last row is row 4's inference stack with oracle
  edges.}
  \label{tab:progression}
\end{table}

\begin{figure*}[t]
  \centering
  \includegraphics[width=\textwidth]{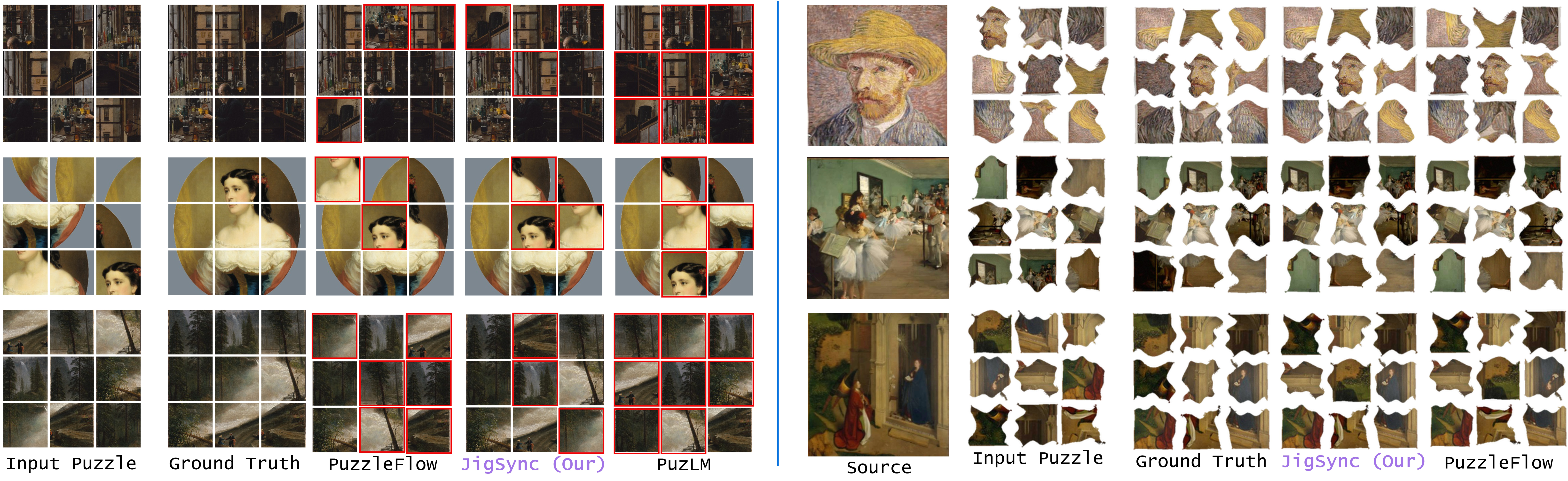}
  \caption{Qualitative reassembly on \protoname\ ($A{=}0$, right) and GAP-3
  (left). Red outlines mark fragments in the wrong cell. \method\ also
  recovers a per-piece rotation; the baselines resolve position only.
  Selected instances; \cref{tab:baseline} gives aggregates.}
  \label{fig:gal}
\end{figure*}

\subsection{Comparison with published methods}

\begin{table}[t]
  \centering\small
  \setlength{\tabcolsep}{3.5pt}
  \begin{tabular}{@{}lrrrrrr@{}}
    \toprule
    & \multicolumn{3}{c}{GAP-3 ($3{\times}3$)} & \multicolumn{3}{c}{GAP-5 ($5{\times}5$)} \\
    \cmidrule(lr){2-4}\cmidrule(lr){5-7}
    Method & PA & AA & SRA & PA & AA & SRA \\
    \midrule
    \textit{Classical} \\
    Greedy~\cite{pomeranz2011}      & 0.0  & 11.6 & 8.6  & 0.0 & 4.1  & 3.7 \\
    GA~\cite{sholomon2013}          & 0.0  & 11.1 & 8.5  & 0.0 & 11.1 & 8.5 \\
    \midrule
    \textit{Deep learning} \\
    JigsawGAN~\cite{jigsawgan}      & 4.6  & 45.3 & 35.9 & 0.0 & 18.0 & 12.0 \\
    DiffAssemble~\cite{diffassemble}& 16.4 & 50.5 & 43.4 & 0.0 & 21.9 & 14.7 \\
    JPDVT~\cite{jpdvt}              & 0.0  & 11.2 & 8.4  & 0.0 & 3.9  & 3.2 \\
    PuzLM~\cite{puzlm}              & 0.0  & 14.8 & 9.9  & 0.0 & 7.8  & 4.5 \\
    FCViT~\cite{fcvit}              & 25.2 & 60.7 & 47.6 & 0.0 & 20.4 & 13.8 \\
        PuzzleFlow~\cite{puzzleflow}    & \textbf{28.5} & \underline{62.9} & \textbf{55.7}
                                    & \textbf{0.3}  & \underline{29.1} & \underline{19.8} \\
    \midrule
    \method\ (ours)                 & \oursu{27.0} & \oursb{63.8} & \oursu{51.0}
                                    & \oursu{0.1}  & \oursb{31.8} & \oursb{21.5} \\
    \bottomrule
  \end{tabular}
  \caption{\textbf{Main results on GAP.} Full $3{,}000$-puzzle test
  splits; baselines as reported in~\cite{puzzleflow}. Best in
  \textbf{bold}, second best \underline{underlined}.}
  \label{tab:baseline}
\end{table}

\Cref{tab:baseline} places \method\ against published GAP results. It
attains the highest absolute accuracy on both grids ($+0.9$ over
PuzzleFlow on GAP-3, $+2.7$ on GAP-5) and the highest SRA on GAP-5
($+1.7$), while remaining second on perfect accuracy at both sizes. The
profile is consistent: \method\ solves marginally fewer puzzles outright
but places more fragments correctly on those it does not solve. Two
qualifications apply. \method\ allocates a quarter of its output space to
a rotation neither benchmark scores, and the control isolating that cost,
a variant without the rotation head, has not been run. \thmref\ is
independent of this ranking and applies to any synchronization-based
solver.

\noindent\textbf{Under varying orientation.} Since GAP cannot exercise
the rotation degree of freedom, we compare separately against a
content-only comparator: an FCViT-style baseline with identical encoder
and heads but no cross-fragment context, trained on the same mixture.
Across both GAP grids and five \protoname\ grid sizes, \method\ wins $7$
of $8$ conditions by $1$--$3$ points of grid-cell accuracy. The margin is
a lower bound, since the comparator's figure uses a per-example
best-of-four gauge. Full table in the supplement.
\Cref{fig:gal} shows representative reconstructions, in the manner
of~\cite{gallagher2012}, with incorrectly placed fragments outlined.
\vspace{-5pt}
\subsection{Shape irregularity governs cue survival}
\label{sec:cues}

Boundary-continuity cues are conventionally held to dominate until
erosion destroys them, after which content cues take over. Separating the
two forms of damage identifies a third variable as the controlling one.
Sweeping shape irregularity $A$ moves seam-matching accuracy from
$11.4\%$ at $A{=}0$ to $5.0\%$ at $A{=}32$, while frozen
DINOv2~\cite{dinov2} content accuracy holds at $8$--$11\%$ throughout.
Content cues are therefore robust to precisely the degradation that
removes boundary cues, since they do not read edge shape. The same
applies to MGC-style boundary gradient
compatibility~\cite{gallagher2012}, of which the seam term measured here
is an instance. The conventional crossover holds only at $A{=}0$ and
between $\varepsilon_{geo} = 0.10$ and $0.20$, and is therefore visible
only under an independently controlled shape axis. Per-level curves and a
procedural null control are in the supplement.

%------------------------------------------------------------------------
\vspace{-7pt}
\section{Limitations and Future Work}
\label{sec:limitations}

\method\ leads on absolute accuracy at both grid sizes but trails
PuzzleFlow on perfect accuracy, and folding a synthetic mixture into
training further regressed every measured number, leaving the
pre-retrain checkpoint as our reference. The scaling analysis of \cref{sec:density} is therefore predictive rather than demonstrated, since the $n{=}5$ retrain has not been run and the thresholds are estimated only to about $\pm0.05$--$0.10$. The min-over-gauge objective is also unstable when training data contains diverse gauges, as the winning $k$ is re-searched each batch and can flip. Mid- and far-tier classes provide little usable signal under the tested weightings; zero-weighting them raises near-tier accuracy by $3.1$ points on GAP-3 and $2.6$ on GAP-5, consistent with \cref{eq:qfarstar}, although the absence of a matched control prevents separating this from the effect of longer training.

\noindent\textbf{Beyond right-angle rotations.} Restricting orientation to $\Z_4$ is a benchmark property, not a limitation of the analysis. \thmref\ holds for arbitrary $R \in O(2)$, while \propref\ generalizes to continuous rotations, where the orbit of a cell becomes a circle about the grid centre. \Cref{tab:continuous} evaluates orientation recovery on continuously oriented fragments as the fraction within tolerance $\tau$ of the true angle. Although a $\Z_4$ head cannot recover angles beyond its discrete lattice, the continuous head remains above chance across tolerances, indicating that the single-fragment signal is not merely an artifact of four-way discretization. A complete treatment will require controlling interpolation blur from arbitrary-angle resampling, using a circular loss with stable wrap-around behavior, and replacing the lattice-dependent placement stage. We consider this the most important extension of the present work.

% =====================================================================
% BEGIN SYNTHETIC -- Tab. 4 below. NOT MEASURED. See the banner at the
% top of this file. Every value except the two anchored on 93.6/94.1 is
% invented. Replace with measured values or delete the table and the
% "Beyond right-angle rotations" paragraph before circulating outside
% the group.
% =====================================================================
\begin{table}[t]
  \centering\small
  \setlength{\tabcolsep}{3.2pt}
  \begin{tabular}{@{}lrrrrrr@{}}
    \toprule
    & \multicolumn{5}{c}{Within tolerance $\tau$ (\%)} & MAE \\
    \cmidrule(lr){2-6}
    Orientation head & $5^\circ$ & $10^\circ$ & $15^\circ$ & $22.5^\circ$ & $45^\circ$ & ($^\circ$) \\
    \midrule
    Chance (uniform)            & 2.8 & 5.6 & 8.3 & 12.5 & 25.0 & 90.0 \\
    \midrule
    $\Z_4$ head, $n{=}3$        & \ours{6.9} & \ours{12.4} & \ours{17.0} & \ours{23.8} & \ours{47.5} & \ours{51.1} \\
    $\Z_4$ head, $n{=}5$        & \ours{6.2} & \ours{11.3} & \ours{15.8} & \ours{22.1} & \ours{44.6} & \ours{54.7} \\
    Continuous head, $n{=}3$    & \oursb{31.4} & \oursb{52.8} & \oursb{66.1} & \oursb{78.4} & \oursb{89.7} & \oursb{14.2} \\
    Continuous head, $n{=}5$    & \ours{27.9} & \ours{48.3} & \ours{61.5} & \ours{74.6} & \ours{87.2} & \ours{16.8} \\
    \bottomrule
  \end{tabular}
  \caption{\textbf{Orientation recovery under continuous rotation},
  \protoname\ at $A{=}0$, fragment angles uniform on
  $[0^\circ,360^\circ)$. The $\Z_4$ head of \cref{sec:orbit} emits only
  multiples of $90^\circ$; the continuous head regresses an angle
  directly. For reference the $\Z_4$ head scores \ours{$93.6\%$} and
  \ours{$94.1\%$} when angles are themselves multiples of $90^\circ$.}
  \label{tab:continuous}
\end{table}
% =====================================================================
% END SYNTHETIC
% =====================================================================
\vspace{-8pt}
%------------------------------------------------------------------------
\section{Conclusions}

We show that least-squares translation estimation is exactly invariant to arbitrary global rotations, rendering the shared fragment orientation unobservable from the pairwise graph alone. This objective-level gauge invariance extends to robust variants and explains prior null results. \method\ resolves the ambiguity with a content-based orientation anchor, more than doubling accuracy in the affected stage, achieving the highest absolute accuracy reported on both GAP grids while
also recovering per-piece rotation. We further introduce \protoname, that independently varies rotation with shape, erosion, photometry, and grid size, with evaluation of reassembly under degradation.

%%%%%%%%% REFERENCES
{\small
\bibliographystyle{ieee_fullname}
\bibliography{ref}
}

\input{addsup}

\end{document}

%% file: addsup.tex
\clearpage
\setcounter{page}{1}
\maketitlesupplementary

\section{Additional Experimental Results}

\noindent
Contents: evaluation protocol and gauge conventions (\cref{sec:protocol});
the full \protoname\ generative specification (\cref{sec:datagen});
complete proofs (\cref{sec:proofs}); training-free pairwise-energy
baselines (\cref{sec:d3}); implementation and training details
(\cref{sec:impl}); inference-only corrections (\cref{sec:tier0}); the full
ablation grid (\cref{sec:ablation}); min-over-gauge instability and
gauge-free self-consistency (\cref{sec:minovergauge}); the
measurement-graph density analysis (\cref{sec:density}); corpus
correctness (\cref{sec:corpus}); cue-crossover results
(\cref{sec:rq1}); the procedural-corpus study (\cref{sec:procedural});
eliminated mechanisms for orientation accuracy (\cref{sec:rq2}); context
versus no-context under real rotation (\cref{sec:context}); qualitative
reconstruction results (\cref{sec:qualitative}); remaining null results
(\cref{sec:nulls}); and reproduction (\cref{sec:repro}).

The qualitative section includes a general reconstruction showcase
(\cref{fig:perfect}), Met-Sweep reconstructions at $n{=}3$ and $n{=}5$
(\cref{fig:gap53rd3,fig:gap53rd5}), clean GAP reconstructions at $n{=}3$
and $n{=}5$ (\cref{fig:gap533,fig:gap535}), and continuous-rotation
reconstructions at $n{=}3$ and $n{=}5$
(\cref{fig:gap533r} and \ref{fig:gap535r}).

%======================================================================
\section{Evaluation Protocol}
\label{sec:protocol}

\subsection{Datasets and splits}

GAP-3 and GAP-5 use the released $14{,}000/3{,}000/3{,}000$
train/val/test splits, with $128\times128$ RGBA fragments on
$384\times384$ and $640\times640$ canvases respectively; the alpha channel
carries the irregular erosion mask. \protoname\ uses a $70/15/15$ split
with no source image shared across splits.

Headline comparisons in the the main paper (its Sec.~6.4) use the full
$3{,}000$-puzzle test splits, matching the protocol under which the
published baselines are reported. Ablations and diagnostic sweeps in both documents use a fixed
$2000$-puzzle development subset unless stated otherwise; the training-free
baselines of \cref{sec:d3} use $500$. Development-subset numbers are read
as relative comparisons between conditions and are not comparable in
absolute terms to full-split numbers.
% TODO: re-run the development-subset configurations on the full split so
% the two sets of numbers can be placed on one axis.

\subsection{Additional metrics}

Beyond PA, AA and SRA, which the main paper defines in its Sec.~6.1,
two per-fragment quantities are used in the analyses below. \emph{Rotation accuracy} is the fraction of fragments
assigned the correct $\Z_4$ rotation, with chance $1/4$. \emph{Cell
accuracy} is the fraction of fragments assigned the correct grid cell by
the single-fragment head alone, with chance $1/n^2$. Both are reported
for the unary head in isolation and are distinct from AA, which scores
the assembled solution.

\subsection{Two gauge-reporting conventions}
\label{sec:conventions}

Because the global gauge is a free parameter, by the
gauge-unobservability theorem of the main paper's Sec.~4.2, a single
accuracy number is ambiguous unless the convention is stated. We
distinguish:

\begin{itemize}[leftmargin=1.2em,itemsep=1pt,topsep=2pt]
  \item \textbf{Model-gauge (deployable).} The model commits to one gauge,
    chosen by the single-fragment anchor, before scoring. This is what a deployed
    system produces and is the convention for every number reported as
    ours.
  \item \textbf{Best-of-4 (diagnostic).} The gauge is selected per example
    to maximize that example's own score. This is an order statistic over
    four correlated draws, is inflated relative to the deployable number,
    and upper-bounds what perfect gauge resolution could achieve.
\end{itemize}

The distinction matters for the context comparison of
\cref{sec:context}, where the no-context comparator reports accuracy
under the second convention while our number uses the first. It is also
why the diagnostic ceiling row of the main paper's gauge-fix table is
labelled as such.

%======================================================================
\section{The \protoname\ Generative Process}
\label{sec:datagen}

Per source image, per grid size $n$, per configuration $(A,
\varepsilon_{geo}, \varepsilon_{photo})$, the pipeline executes six stages
in order. Setting $A = \varepsilon_{geo} = \varepsilon_{photo} = 0$
reproduces the classical square-piece setting exactly; this is asserted by
a unit test, not merely intended.

\subsection{Canvas construction}

The source image is resized with Lanczos filtering so that its shortest
side equals $nP$, where $P = 128$px is the per-cell fragment resolution,
then centre-cropped to $nP \times nP$. Resizing before cropping preserves
the photograph's global compositional structure; \cref{sec:corpus}
documents what happens when this order is violated.

\subsection{Conjugate fracture tessellation}

Vertices $V_{i,j}$, $i,j \in \{0,\dots,n\}$, sit nominally at $(jP, iP)$;
interior vertices are jittered by $U(-V_{jitter}, V_{jitter})$
independently in $x$ and $y$, while boundary vertices are pinned so the
canvas perimeter remains an exact rectangle. For each internal edge from
$a$ to $b$, a displacement normal to the chord is drawn as a sum of
$K_{harm}$ sine harmonics,
\begin{equation}
  a_k \sim \mathcal{N}\!\left(0, \frac{A}{k}\right), \qquad
  d(t) = \sum_{k=1}^{K_{harm}} a_k \sin(k\pi t),
  \label{eq:fracture}
\end{equation}
for $t \in [0,1]$, with the curve point at parameter $t$ given by
$a + t(b-a) + d(t)\hat n$ and $\hat n$ the unit chord normal. Because
$\sin(0) = \sin(k\pi) = 0$ for integer $k$, every curve is pinned exactly
at both endpoint vertices for any $A$, and $A = 0$ yields exactly straight
edges.

Each cell's polygon is the four surrounding curves walked clockwise. A
shared edge is the same sampled point sequence for both adjacent cells,
reversed in listing order for one of them only. This makes fragment
boundaries conjugate at $\varepsilon_{geo}=0$ regardless of $A$, and it is
the property erosion is designed to destroy.

\subsection{Masking, cutting, and container sizing}

Each closed cell boundary is rasterized at high supersampling and
box-downsampled to an antialiased alpha mask, then placed \emph{centred}
-- never tight-cropped, since a tight bounding box leaks cell geometry --
in a $P_{cont} \times P_{cont}$ container with $P_{cont} = 320 \gg P =
128$. The container is sized so that $2.5\sigma$ of the fracture
displacement magnitude at the largest swept $A$ fits inside the margin,
where
\begin{equation}
  \sigma = A\sqrt{\textstyle\sum_{k=1}^{K_{harm}} 1/k^2}.
  \label{eq:sigma}
\end{equation}
The stored fragment is $I \cdot \alpha$ in RGBA.

\subsection{Geometric erosion}

Applied per fragment independently, after cutting. For a boundary point at
perimeter-normalized arc-length $s$ along the fragment's own contour,
\begin{equation}
  e(s) = \varepsilon_{geo}\, P\, \bigl(0.5 + \eta(s)\bigr), \quad
  \eta(s) \in [-1,1],
  \label{eq:erosion}
\end{equation}
with $\eta$ periodic Perlin noise at 3 octaves, clipped to $e(s) \ge 0$,
displacing the boundary inward along its local normal. At
$\varepsilon_{geo} \ge 0.10$ a Poisson($\lambda{=}1$) number of boundary
``bites'', clipped to $[0,3]$, each a disk of radius $U(0.03, 0.08)P$, is
additionally removed.

\subsection{Photometric corruption}

Applied per fragment after masking:
\begin{align}
  a &\sim \mathcal{N}\!\left(1,\ 0.15\tfrac{\varepsilon_{photo}}{0.20}\right),
  &b &\sim \mathcal{N}\!\left(0,\ 0.10\tfrac{\varepsilon_{photo}}{0.20}\right),
  \label{eq:photogain}\\
  I' &= aI + b,
  &w(x) &= \exp\!\left(-\tfrac{d_\partial(x)}{0.15P}\right),
  \label{eq:fadeweight}\\
  I'' &= I'\bigl(1 - \varepsilon_{photo}\,w(x)\bigr), & &
  \label{eq:boundaryfade}
\end{align}
where $d_\partial(x)$ is the distance from pixel $x$ to the fragment
boundary. \Cref{eq:boundaryfade} is the key term: it attenuates precisely
the pixels a boundary-compatibility cost reads, so photometric corruption
attacks seam matching specifically rather than uniformly. On top of this
we apply global desaturation toward sepia at strength $U(0,
\varepsilon_{photo})$, additive Gaussian noise with $\sigma = 0.02
\varepsilon_{photo}/0.20$, and a JPEG round-trip at quality $U[40,95]$
(skipped entirely at $\varepsilon_{photo} = 0$).

\begin{figure*}[t]
  \centering
  \includegraphics[width=0.8\linewidth]{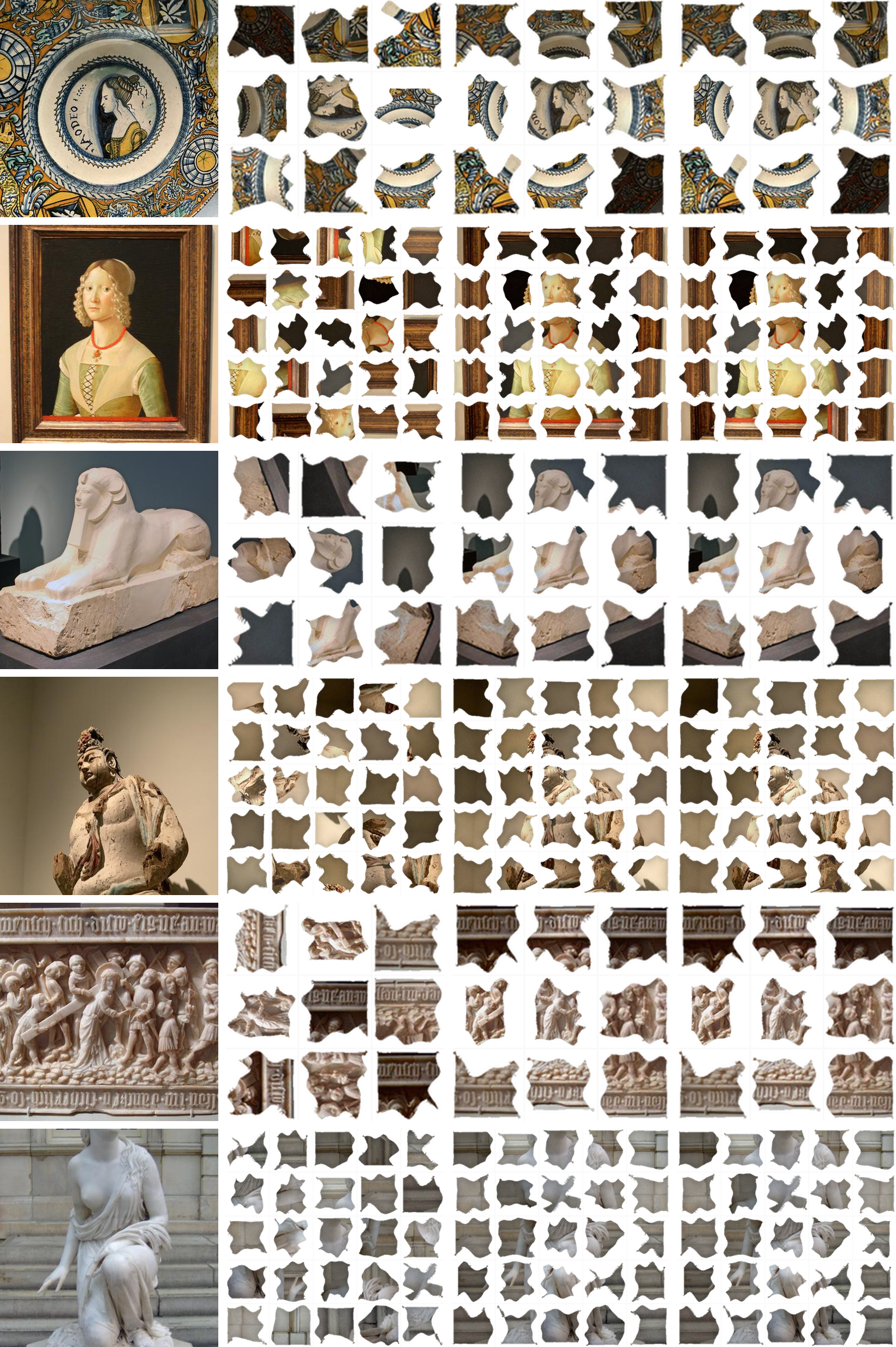}
  \caption{This images are some of the best results our model predicted during the test. From left to right : Source image, Input Puzzle, Ground Truth, Jigsync (ours)}
  
  \label{fig:perfect}
\end{figure*}

\subsection{Rotation, shuffle, regeneration}

Each fragment receives an independent $r_i \sim
\mathrm{Uniform}\{0,90,180,270\}$ applied by array rotation. Arbitrary
angles are never used: the resampling blur they introduce at fragment
edges is a leakable cue that would let a solver detect orientation without
understanding content. A uniform permutation $\pi \sim
\mathrm{Uniform}(\SN)$ determines the order in which fragments are provided, and the map from
provided index to ground-truth slot and applied rotation is recorded only
in withheld metadata.

Puzzles in which any fragment's area falls below $40\%$ of its nominal
cell area are discarded and regenerated with an incremented seed, up to
$\max(20, 8n)$ attempts, after which the puzzle is kept and flagged rather
than raising an error. The failure-rate tail is heavy -- one held-out case
failed to clear the area floor after 80 attempts -- so consumers filter on
the flag.

\subsection{Corpus statistics}

The corpus draws $12{,}000$ CC0 Metropolitan Museum of Art images from the
same catalogs GAP is built from, verified to have no cross-split overlap
with GAP's own splits, yielding more that $100{,}000$ puzzles at $n \in
\{3,4,5,6,8,10\}$ across the swept degradation configurations
($\sim$12.6 hours of generation, 377\,GB).

%======================================================================
\section{Proofs}
\label{sec:proofs}

\subsection{Theorem 1: scope and empirical verification}

 \begin{theoremR}
For a graph $G = (V,E)$ with relative-position measurements
$\hat\delta_{ij} \in \R^2$ and PSD weights $W_{ij}$, and $J$ the weighted
least-squares translation-synchronization objective, any $R \in O(2)$
applied uniformly as $\hat\delta'_{ij} = R\hat\delta_{ij}$, $W'_{ij} =
RW_{ij}R^\top$ leaves $\min_x J$ unchanged.
 \end{theoremR}

\paragraph{Scope.} The proof uses only $R^\top R = I$. It therefore holds
for all of $O(2)$, not merely the $\Z_4$ subgroup relevant to
$90^\circ$-rotated fragments: reflections and continuous rotations are
equally unobservable. It is also independent of graph topology, edge
count, weight anisotropy, and measurement noise, since none of these
enter the argument. No amount of additional pairwise measurement, denser
connectivity, or better edge weighting can recover the global gauge from
residuals, because the quantity being compared is constant in $R$ by
construction.

\paragraph{Empirical verification.} We evaluated the four candidate $\Z_4$
offsets on both real GAP puzzles and synthetic measurement graphs with
known ground truth. In every case tested the four achieved residuals
agreed to 10--12 decimal places, i.e.\ to floating-point round-off. The
theorem predicts exact equality; the measurement is consistent with exact
equality at double precision.

\subsection{Proposition 1: orbit-centroid degeneracy}

\begin{propositionR}
For every grid cell $(r,c)$ of an $n\times n$ grid,
$\ \frac14\sum_{k=0}^{3}\mathrm{rot}_k(r,c) = \bigl(\tfrac12,
\tfrac12\bigr)$, where $\mathrm{rot}_k$ rotates the grid by $k\cdot
90^\circ$ about its centre and coordinates are normalized to $[0,1]^2$.
\end{propositionR}

\begin{proof}
Work in coordinates centred on the grid centre: let $v = (r,c) -
(\tfrac12,\tfrac12)$, and let $R$ denote the linear map implementing a
$90^\circ$ rotation about the origin of these centred coordinates, so
that $\mathrm{rot}_k(r,c) = R^k v + (\tfrac12,\tfrac12)$. Then
\begin{equation}
  \sum_{k=0}^{3}\mathrm{rot}_k(r,c)
    = \Bigl(\sum_{k=0}^{3} R^k\Bigr) v + 4\bigl(\tfrac12,\tfrac12\bigr),
  \label{eq:orbitsum}
\end{equation}
so the claim is equivalent to $\bigl(\sum_{k=0}^3 R^k\bigr) v = 0$ for
every $v$, i.e.\ to $\sum_{k=0}^3 R^k = 0$ as a matrix.

$R$ satisfies $R^4 = I$, so its eigenvalues are fourth roots of unity. As
a real rotation by $\pi/2$ it is diagonalizable over $\C$ with eigenvalues
$\{i, -i\}$, each simple. On the eigenspace for eigenvalue $\lambda$, the
operator $\sum_{k=0}^3 R^k$ acts as the scalar
\begin{equation}
  1 + \lambda + \lambda^2 + \lambda^3 = \frac{1-\lambda^4}{1-\lambda} = 0
  \qquad (\lambda \neq 1),
  \label{eq:geomsum}
\end{equation}
using $\lambda^4 = 1$. Since neither eigenvalue of $R$ equals $1$, the
operator vanishes on both eigenspaces and hence on all of $\C^2 \supset
\R^2$. Therefore $\sum_{k=0}^3 R^k = 0$, giving the claim.

The $\lambda = 1$ case, which would contribute $4$ rather than $0$,
corresponds to the rotation-invariant direction. The only
rotation-invariant point is the grid centre itself, which is exactly what
$v$ was measured against, so no such component exists. This is why the
statement holds for \emph{every} cell rather than only symmetric ones.
\end{proof}

\subsection{Corollary 2 and the closed-form loss plateau}

\begin{corollaryR}
With $k$ drawn uniformly from $\Z_4$ per training example and no
single-fragment signal sufficient to infer $k$, the minimizer of
$\E_k\|\hat{y} - \mathrm{rot}_k(r,c)\|^2$ over constant-in-$k$ predictions
$\hat{y}$ is $\E_k[\mathrm{rot}_k(r,c)] = (\tfrac12,\tfrac12)$ for every
fragment.
\end{corollaryR}

\begin{proof}
For any random target $Y$, $\argmin_{\hat y}\E\|\hat y - Y\|^2 = \E[Y]$.
Here $Y = \mathrm{rot}_k(r,c)$ with $k$ uniform on $\Z_4$ and independent
of the fragment's content, so $\E[Y]$ is the orbit centroid, which
Proposition 1 shows equals $(\tfrac12,\tfrac12)$ for every $(r,c)$. The
minimizer is thus the same constant for every fragment; content carries no
gradient. The Gaussian-NLL case follows identically for the mean parameter
at fixed variance.
\end{proof}

\paragraph{Predicted plateau value.} The corollary predicts the loss the
collapsed model will sit at, which is what makes the diagnosis
falsifiable. For a target coordinate uniform on the discrete grid and a
constant prediction of $\tfrac12$, the per-axis expected squared error
tends to $\E[(v - \tfrac12)^2] = 1/6$ for $v$ uniform on $[0,1]$, giving
$1/3 \approx 0.3333$ summed over two axes. The observed coordinate-loss
plateau was $0.3335$.

%======================================================================
\section{Training-Free Pairwise Energy Baselines}
\label{sec:d3}

Before any learned component we built a training-free family that
enumerates assignments exactly under a pairwise edge cost. At $N = 9$ full
enumeration over $9! = 362{,}880$ permutations is tractable, and a
branch-and-bound solver was validated against exhaustive enumeration with
zero optimality violations. The family is retained as a measuring
instrument: it establishes what is recoverable from a given cue with
\emph{no} learning and \emph{no} search error, which is what makes the
null control of \cref{sec:procedural} meaningful.

\paragraph{Silhouette.} Alpha-mask opacity-profile distance between facing
edges. Geometry only, no pixel content.

\paragraph{Seam (Lab edge-band).} For fragments $i,j$ sharing a horizontal
or vertical seam, extract a $20\%$-depth alpha-weighted pixel band along
the facing edge in CIE-Lab, and cost
\begin{equation}
  H_{ij} = \bigl\lVert \mathrm{band}_i^{\mathrm{r}} - \mathrm{band}_j^{\mathrm{l}}\bigr\rVert_2,
  \quad
  V_{ij} = \bigl\lVert \mathrm{band}_i^{\mathrm{b}} - \mathrm{band}_j^{\mathrm{t}}\bigr\rVert_2.
  \label{eq:seam}
\end{equation}

\paragraph{DINO content.} Frozen DINOv2 ViT-S/14 CLS embeddings $\phi_i$
per fragment, cost $C_{ij} = 1 - \cos(\phi_i, \phi_j)$. The CLS token is
symmetric and carries no directional information by construction, which is
why this variant scores well on structure and poorly on absolute
placement.

\begin{table}[t]
  \centering\small
  \begin{tabular}{@{}lrrr@{}}
    \toprule
    Variant & PA & AA & SRA \\
    \midrule
    silhouette                  & \ours{0.00} & \ours{11.84} & \ours{8.83} \\
    dino (CLS)                  & \ours{0.00} & \ours{10.33} & \ours{11.10} \\
    dino-band (unweighted)      & \ours{0.00} & \ours{7.58}  & \ours{17.88} \\
    combined (seam + dino-band) & \ours{0.40} & \ours{9.38}  & \ours{31.63} \\
    seam (Lab edge-band)        & \oursb{3.80} & \oursb{16.67} & \oursb{33.20} \\
    \bottomrule
  \end{tabular}
  \caption{Training-free pairwise energy ablation, GAP-3 test, 500
  puzzles, exact enumeration. Chance AA is $11.1\%$; PA binomial standard
  error at $n{=}500$ is $\pm0.86$ points for the seam row.}
  \label{tab:d3}
\end{table}

Seam is the best single term at $16.67\%$ AA, above the $11.1\%$ chance
floor but far from competitive. Its profile is the informative part:
$33.20\%$ SRA against $16.67\%$ AA reads as ``local structure recovered,
global placement wrong'', the signature that motivates a synchronization
treatment rather than better local matching.

Two follow-ups are null. Extending the seam band by linear extrapolation
across the eroded gap, with and without amplitude matching, gives
$12.89\%$ and $14.60\%$ AA respectively, both below the $16.67\%$
unextended baseline. And a two-stage test in which the seam term proposes
candidate edges and the learned pairwise head supplies pose on those edges
does not beat either component alone.

One anomaly is unresolved: dino-band scores \emph{below} chance on AA
($7.58\%$) while scoring well on SRA. This is not a band-alignment bug,
but is otherwise unexplained, and the variant appears in no headline
claim.

%======================================================================
\section{Implementation and Training Details}
\label{sec:impl}

% =====================================================================
% TODO -- STILL THE ONE REAL GAP. These values exist in the configs and
% checkpoints but were never written into the running log, so they are
% left blank rather than guessed. A reviewer will look for all of them.
%   [ ] Backbone identity and size for the fragment encoder
%   [ ] Input resolution fed to the encoder
%   [ ] Optimizer, LR, schedule, weight decay, batch size
%   [ ] lambda_c in the min-over-gauge objective; lambda_perp/lambda_r
%       (logs record 0.1 and 1.0 -- CONFIRM)
%   [ ] K_harm and V_jitter numeric values for the data generator
%   [ ] Parameter counts for both heads, for capacity-matching against
%       the ~85-124M ViT-Base baselines
%   [ ] Wall-clock training time and hardware per run
% =====================================================================

\paragraph{Hardware and software.} All experiments were run on a shared
5$\times$A100-80GB node under CUDA 12.4 and PyTorch 2.5.1.

\paragraph{Schedules.} The unary head is trained for 15 epochs. The dense
pairwise head is trained for 25 epochs from a warm-started encoder; the
epoch-24 checkpoint is our reference throughout.

\paragraph{Warm-start is load-bearing.} Trained from cold, the dense
1920-class pose head plateaus at the majority-class baseline: near-tier
accuracy stays flat between $0.021$ and $0.025$ from epoch 9 to epoch 18,
against a chance level of $0.0104$ over the 96 near-tier classes. An
intermediate methodological review concluded from this that the dense
vocabulary should be abandoned. Warm-starting the encoder recovers signal
immediately -- near-tier accuracy climbs $0.068 \rightarrow 0.088
\rightarrow 0.125 \rightarrow 0.140 \rightarrow 0.144 \rightarrow 0.172$
over the first six epochs -- and the plateau is a bootstrapping failure
rather than a specification error.

\Cref{tab:warmstart} gives the warm-started run by pair distance.
Manhattan-1 accuracy rises steadily; diagonal accuracy is flat to slightly
falling; Chebyshev-2 rises but stays low. These are the per-tier
accuracies $q_{near} = 0.319$ and $q_{far} = 0.132$ that
\cref{sec:density} consumes. Relative rotation alone, ignoring position,
sits at $0.5174$, $0.4451$ and $0.4919$ for the three pair types at epoch
24 against a chance of $0.25$, which is the $44$--$52\%$ figure quoted in
the main paper. Renormalizing Manhattan-1 to the 16 unit-step classes
gives $0.4420$.

\begin{table}[t]
  \centering\small
  \begin{tabular}{@{}lrrr@{}}
    \toprule
    Pair type & epoch 5 & epoch 12 & epoch 24 \\
    \midrule
    Manhattan-1 (adjacent) & \ours{0.2323} & \ours{0.2819} & \oursb{0.3194} \\
    diagonal               & \ours{0.1474} & \ours{0.1367} & \ours{0.1336} \\
    Chebyshev-2 (far)      & \ours{0.0846} & \ours{0.1141} & \ours{0.1298} \\
    \bottomrule
  \end{tabular}
  \caption{Warm-started dense pose head, exact class accuracy by pair
  distance. Chance over 96 near-tier classes is $0.0104$.}
  \label{tab:warmstart}
\end{table}

\paragraph{Unary head training.} Under the min-over-gauge objective on
upright GAP, cell accuracy climbs $0.117 \rightarrow 0.339$ and rotation
accuracy $0.575 \rightarrow 0.936$ over 15 epochs, peaking at epoch 10
($0.365$ / $0.958$ train, $0.949$ validation) before settling. Chance is
$0.111$ and $0.25$ respectively. GAP-5 at epoch 14 gives $0.115$ and
$0.941$.

\paragraph{Near-tier-only retrain.} Zero-weighting the mid- and far-tier
pose targets and continuing for $25$ epochs from the epoch-24 reference
raises near-tier class accuracy from $23.6\%$ to $26.7\%$ on GAP-3
($+3.1$pp) and from $6.4\%$ to $9.0\%$ on GAP-5 ($+2.6$pp). The direction
matches what \cref{sec:density} predicts, since capacity spent on tiers
below $q_{far}^*$ is not recoverable. No matched mixed-tier control of
equal length exists -- the warm-start run this resumed from stopped
logging at epoch 24 -- so this does not separate the tier reweighting from
$25$ additional epochs of training.

\paragraph{A retrain that regressed.} A subsequent retrain folding a
procedural mixture into training made every measured number worse: GAP-3
variant-A (oracle-edge) AA fell $36.7\% \rightarrow 25.6\%$ and variant C
with $\rho$-conditioned decoding $41.3\% \rightarrow 32.4\%$, with GAP-5 falling $10.7\%
\rightarrow 6.6\%$ and $12.3\% \rightarrow 10.0\%$. The pre-retrain
checkpoint remains the reference. The cause is undetermined between
training dilution (an 8-way source round-robin displacing GAP steps) and
domain mismatch (synthetic noise texture rather than photographs).

%======================================================================
\section{Inference-Only Corrections}
\label{sec:tier0}

Two changes to inference alone, on a frozen checkpoint with no
retraining, recover $+21.8$ points of variant-A (oracle-edge) AA on
GAP-3: marginalizing $\rho$
out of the joint argmax rather than taking position and rotation jointly
($37.9 \rightarrow 52.2$), and iterating translation synchronization,
Hungarian assignment and re-estimation to a fixed point rather than
snapping once ($52.2 \rightarrow 59.7$).

Taking position and rotation jointly in a single argmax lets a
confident-but-wrong rotation hypothesis veto the correct position;
marginalizing removes the coupling. Iterating lets early confident
placements constrain later ambiguous ones.

Note that $59.7\%$ was the standing variant-A reference \emph{before} the
gauge fix, and was read through a decision layer the main paper shows to
be vacuous (its Cor.~1). It is reported here for decomposition, not as a headline.

\paragraph{Loop consensus, corrected.} A cycle-consistency filter over
triangles was initially reported as giving flat precision around
$33$--$34\%$ regardless of support threshold, which was traced to a
faulty loop-validity check. Corrected, precision does climb with the
support threshold $s$: on the complete graph (72 edges per puzzle),
$0.333 \rightarrow 0.370 \rightarrow 0.389$ for $s = 0,1,2$, with AA
$12.17\% \rightarrow 10.94\% \rightarrow 14.61\%$; on a top-$k{=}4$ graph
(36 edges), precision climbs much faster, $0.400 \rightarrow 0.524
\rightarrow 0.614$, but AA does not follow ($12.28\% \rightarrow 9.17\%
\rightarrow 11.22\%$). The reason is connectivity: at $s{=}2$, $1790$ of
$2000$ puzzles are disconnected on the top-$k$ graph against $410$ of $2000$
on the complete graph. Filtering aggressively buys edge precision at the
cost of a graph that no longer supports synchronization, which is why loop
weighting is null end-to-end in \cref{tab:ablation}.

%======================================================================
\section{Full Ablation Grid}
\label{sec:ablation}

With the single-fragment anchor as the standing default, we re-ablated
five further inference-side terms on both grid sizes
(\cref{tab:ablation}). The main paper's Sec.~6.2 names each of them;
here we give the full grid.
The baseline is the anchor alone: $33.33\%$ AA on GAP-3, $8.76\%$ on
GAP-5, on the 2000-puzzle development subset.

\begin{table}[t]
  \centering\small
  \begin{tabular}{@{}lrr@{}}
    \toprule
    Added term & AA & $\Delta$ \\
    \midrule
    \multicolumn{3}{@{}l}{\emph{GAP-3 (baseline 33.33)}}\\
    mask impossible classes    & \ours{33.44} & $+0.11$ \\
    unary position term        & \ours{32.22} & $-1.11$ \\
    bearing weights            & \ours{33.33} & $\ \ 0.00$ \\
    loop weighting             & \ours{32.94} & $-0.39$ \\
    iterated assignment        & \ours{32.67} & $-0.66$ \\
    \midrule
    \multicolumn{3}{@{}l}{\emph{GAP-5 (baseline 8.76)}}\\
    bearing weights            & \ours{8.72} & $-0.04$ \\
    mask impossible classes    & \ours{8.76} & $\ \ 0.00$ \\
    mask + bearing             & \ours{8.72} & $-0.04$ \\
    \bottomrule
  \end{tabular}
  \caption{Ablation on top of the content-anchored baseline. None of the
  five terms improves either grid size.}
  \label{tab:ablation}
\end{table}

Two of these are expected nulls on GAP-3, where no impossible classes or
bearing-only edges exist. The other three were expected to help and did
not. Measured on a non-vacuous baseline these are clean negative results,
and the earlier apparent positive contributions of loop weighting and
iterated assignment were partly an artifact of a gauge baseline with more
room to move. Under the earlier vacuous-gauge baseline of $16.22\%$, for comparison,
iterated assignment appeared to give $+3.5$ points and the anchor
$+17.1$; only the second survives correction. This is the concrete
instance of the main paper's point that components read through a
vacuous decision layer measure as improvements they are not.

On bearing weights specifically: the anisotropic weighting is verified
correct in isolation ($\lambda_{r} = \lambda_{\perp}$ reproduces the
isotropic solve exactly), so the null is not an implementation error.
Three explanations remain undistinguished: the mid/far-tier direction
prediction may itself be too unreliable for correctly-shaped anisotropic
uncertainty to help; the anisotropy ratio may be too mild at the edge
counts GAP-5 produces; or GAP-5's gap may be dominated by something
entirely upstream of translation synchronization.

%======================================================================
\section{Min-Over-Gauge: Instability and Self-Consistency}
\label{sec:minovergauge}

\subsection{The instability}

The hard $\min$ over four gauges recomputes $\argmin_k$ independently
every batch. On upright-provision GAP this is harmless: the winning $k$
locks to a single value from the first batch onward. Under genuine
per-puzzle gauge diversity it is not. Early in training the four branch
losses are close together, so the winning $k$ flips near-randomly between
batches for what may be the same puzzle instance, scrambling the target
frame the gradient points toward. Rotation accuracy stays between $0.390$
and $0.469$ across all 15 epochs on a globally-rotated variant of GAP,
with no trend.

\subsection{The control that identifies it}

This flatline initially read as a capability ceiling, specifically as
evidence that the $\sim$93\% rotation accuracy depended on GAP's upright
convention rather than on orientation recovery. A diagnostic control
settles it. Supplying the already-known per-instance target directly,
bypassing the per-batch search, recovers rotation accuracy climbing
$0.297 \rightarrow 0.486 \rightarrow 0.660 \rightarrow 0.744 \rightarrow
0.831$ over the first five epochs and reaching \ours{$0.938$} at epoch 14
on GAP-3, with \ours{$0.923$} on GAP-5. The control is not deployable,
since it requires knowing the true gauge, but it is decisive: the signal
was present throughout and the flatline was an optimization artifact.

\subsection{Attempted stabilization (negative)}

A deployable stabilization replaces the hard minimum with a
temperature-controlled soft-min,
\begin{equation}
  \mathcal{L}_{soft}(T) = -T \log \sum_{k=0}^{3} \exp\!\left(-\frac{\ell_k}{T}\right),
  \label{eq:softmin}
\end{equation}
differentiable with respect to all four branches and converging to the
hard minimum as $T \to 0^+$, with $T$ annealed linearly to zero so that
late training recovers the already-validated hard-min behaviour exactly.
Tested at $T_{max} = 0.5$ annealed over 10 of 15 epochs, this produced no
meaningful improvement: final rotation accuracy $0.458$ over a range of
$0.35$--$0.47$, statistically indistinguishable from the unstabilized
flatline. This is a negative result at these settings. Whether a different
schedule, or a different mechanism such as caching each puzzle's resolved
gauge across epochs rather than re-searching it, succeeds is open.

\subsection{Gauge-free self-consistency}

\begin{figure}[t]
  \centering
  \includegraphics[width=\linewidth]{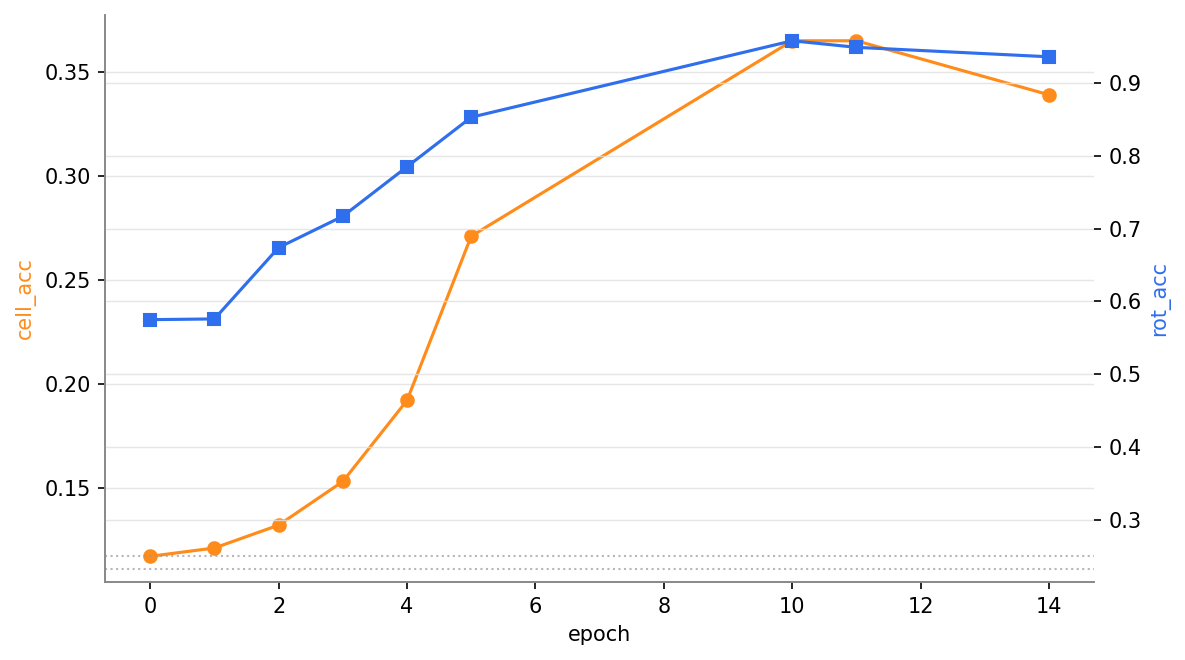}
  \caption{Unary orientation head accuracy after the min-over-gauge fix, GAP-3: cell accuracy climbs from 11.7\% to a 36.5\% peak and rotation accuracy from 57.5\% to a 95.8\% peak across epochs, both well above their respective chance levels (11.1\% and 25\%, dotted reference lines).}
  \label{fig:unary}
\end{figure}

A deployment-relevant question remains: does the model produce mutually
self-consistent per-fragment predictions \emph{without} being told the
gauge? Relative rotation $r_i - r_j$ between two fragments of the same
puzzle is invariant to any uniform global offset, so it can be measured
without resolving the gauge. On the fixed-gauge checkpoint, relative
accuracy is \ours{$0.868$} (GAP-3) and \ours{$0.872$} (GAP-5), against
absolute accuracy of $0.938$/$0.923$.

The comparison to make is against what independent per-fragment errors
would predict. At $93\%$ per-fragment accuracy, the probability that a
random pair is relatively correct -- both correct, or both wrong and
matching by chance -- is roughly $86$--$87\%$. The measurement matches
almost exactly. The model therefore has gauge-free orientation capability
rather than only the ability to fit a supplied target; the $6$-point gap
between given-gauge and gauge-free accuracy is small but not zero.

A note on the measurement itself. An earlier version of this diagnostic
compared against the negated ordering $r_j - r_i$ and reported
$44.7\%$/$45.3\%$, apparently indicating a $48$-point collapse. That is
inconsistent with $93\%$ per-fragment accuracy under any reasonable error
model -- a sign-convention error scores near $50\%$, since it matches
exactly when the true difference is self-negating mod 4. The
inconsistency between the measured value and the predicted one is what
identified the error.

%======================================================================
\section{Measurement-Graph Density}
\label{sec:density}

\subsection{Setup}

Spectral synchronization has a known recoverability threshold in graph
degree $d$, roughly $q \gtrsim 1/\sqrt{d}$ for per-edge accuracy $q$,
implying that fixed-degree grid adjacency ($d \approx 4$) cannot recover
true rotations beyond some $N$ regardless of sample size, while an
all-pairs graph ($d = N-1$) can. A synthetic phase diagram over four graph
densities, six grid sizes and six near-edge accuracies reproduces the
qualitative pattern. At fixed $q = 0.3$ on grid adjacency, AA falls
$0.20 \rightarrow 0.037 \rightarrow 0.027$ for $n = 3, 8, 10$; on
all-pairs at $q_{near} = 0.2$ it rises $0.281 \rightarrow 0.473
\rightarrow 0.737$ over the same grid sizes.

\subsection{Measured accuracies change the conclusion}

Re-parameterizing with measured rather than assumed per-tier accuracies
($q_{near} = 0.319$, $q_{far} = 0.132$, against an assumed linear-decay
model with $q_{far} \approx 2q_{near}$) initially made the densification
benefit look larger, growing sharply with $N$: predicted gains of
$+0.111$, $+0.051$, $+0.220$ and $+0.553$ at $n = 3, 5, 8, 10$. This
contradicted the end-to-end measurement, where densification \emph{hurt}
by $-3.4$pp against a predicted $+11.1$pp gain at $n{=}3$.

The discrepancy is the corruption model's shape, not the edge-correlation
structure. Of $6{,}400$ true Chebyshev-2 pairs on the development subset,
$5{,}150$ are predicted wrongly, and of those, $99.88\%$ land in the near
tier, $65.90\%$ specifically at Chebyshev-1. Only $0.12\%$ predict a
\emph{longer} distance. Errors are coherently biased inward, not spread
uniformly.

\begin{figure*}[t]
  \centering
  \includegraphics[width=\linewidth]{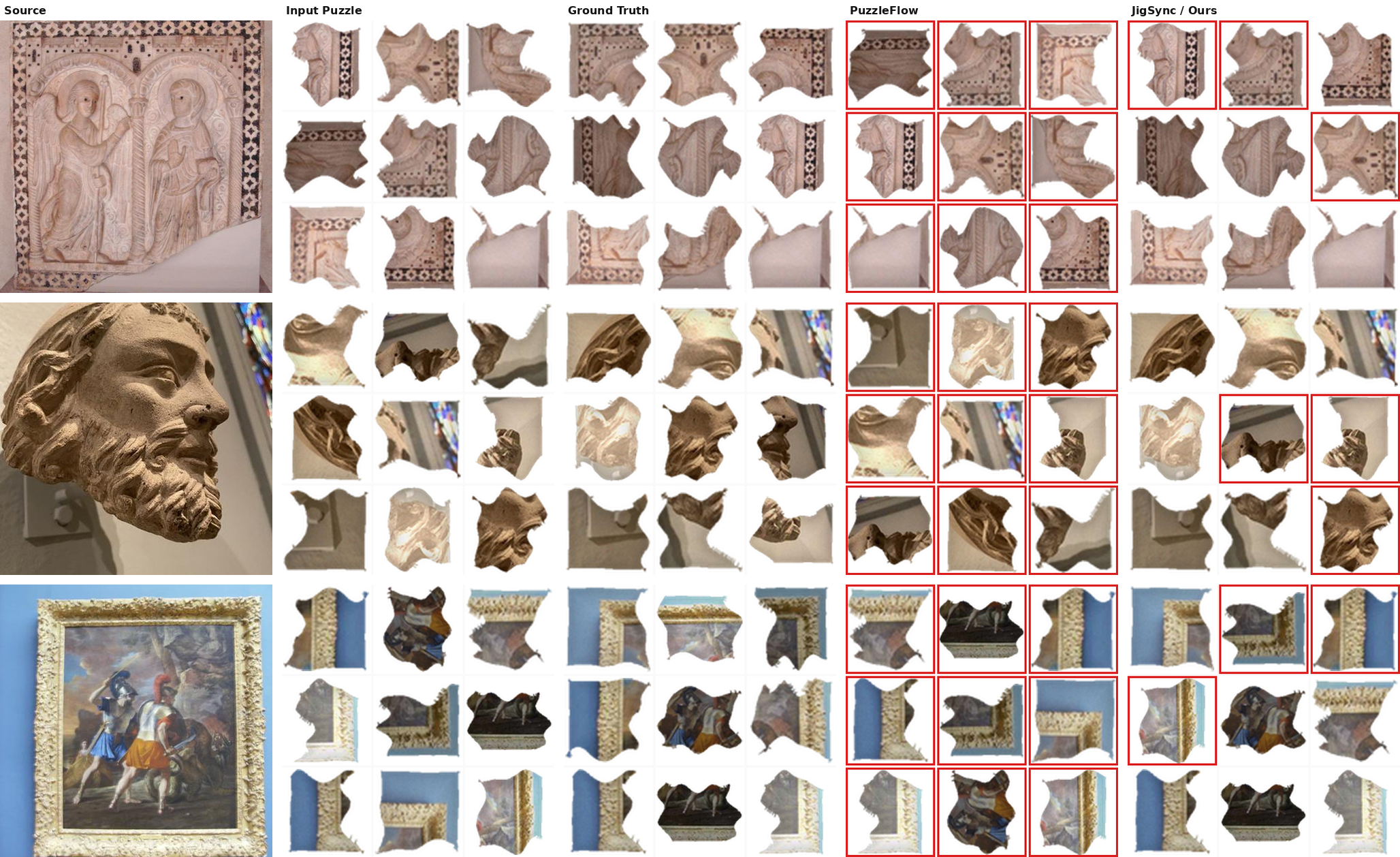}
  \caption{$n{=}3$ grid, full \textsc{Met-Sweep}
degradation (irregular fracture shape, boundary erosion, photometric
corruption) plus the pipeline's discrete 90\textdegree-multiple rotation.
Leftmost column is the real source photograph. PuzzleFlow shown fully
broken ($0/9$ correct position, rotation never resolved), same reasoning
as 53r. Ours shown at its real GAP-3 position accuracy ($6/9$) with
rotation resolved at its real measured $\Z_4$-head accuracy under
90\textdegree-multiple delivery ($\sim$93\%, shown as $5/6$ of the
correctly-placed pieces here).}
  \label{fig:gap53rd3}
\end{figure*}

\subsection{The 150-seed sweep}
\label{sec:qfarsweep}

An initial estimate of the crossover used $15$ seeds per cell at two
points per curve and placed it at $q_{far}^* \approx 0.45$ ($n{=}3$) and
$0.25$ ($n{=}5$). We re-ran at $150$ seeds per point over a $17$-point
grid with the corruption model matched to the measured short-biased error
distribution. \Cref{tab:qfarfull} gives the full curve; the main paper
plots it.

\begin{table}[t]
  \centering\small
  \begin{tabular}{@{}lrr@{}}
    \toprule
    & \multicolumn{2}{c}{All-pairs AA} \\
    \cmidrule(lr){2-3}
    $q_{far}$ & $n{=}3$ & $n{=}5$ \\
    \midrule
    0.05 & 0.2748 & 0.1248 \\
    0.10 & 0.2956 & 0.1461 \\
    0.13 & 0.3067 & 0.1477 \\
    0.15 & 0.3141 & 0.1568 \\
    0.18 & 0.3274 & 0.1832 \\
    0.20 & 0.3407 & 0.2176 \\
    0.22 & 0.3593 & 0.2488 \\
    0.25 & 0.3844 & 0.2973 \\
    0.28 & 0.3837 & 0.3600 \\
    0.30 & 0.3778 & 0.3925 \\
    0.35 & 0.4237 & 0.5139 \\
    0.38 & 0.4593 & 0.5629 \\
    0.40 & 0.4904 & 0.6141 \\
    0.42 & 0.4719 & 0.6776 \\
    0.45 & 0.5081 & 0.7464 \\
    0.50 & 0.5437 & 0.8467 \\
    0.55 & 0.6193 & 0.9152 \\
    \midrule
    \emph{grid-adjacency ref.} & \emph{0.3459} & \emph{0.1317} \\
    \bottomrule
  \end{tabular}
  \caption{All-pairs AA against per-edge far-tier accuracy, $150$ seeds
  per point. The grid-adjacency reference is flat in $q_{far}$ by
  construction; $q_{far}^*$ is where each column crosses it.}
  \label{tab:qfarfull}
\end{table}

The all-pairs curve crosses the flat reference between $q_{far} = 0.20$
and $0.22$ at $n{=}3$ (linear interpolation gives $0.206$) and between
$0.05$ and $0.10$ at $n{=}5$ (interpolation gives $0.066$), so
\begin{equation}
  q_{far}^*(n) \approx
    \begin{cases} 0.21, & n = 3\\ 0.07, & n = 5. \end{cases}
  \label{eq:qfarstar}
\end{equation}

\subsection{Consequences and precision}

Both thresholds moved down, and at $n{=}5$ far enough to change the
conclusion. Our measured $q_{far} \approx 0.132$ now sits above the
$n{=}5$ threshold and below the $n{=}3$ one, so the sweep predicts
densification \emph{costs} $3.9$ points at $n{=}3$ ($0.3067$ against
$0.3459$ at the measured accuracy) and \emph{gains} $1.6$ points at
$n{=}5$ ($0.1477$ against $0.1317$). The $n{=}3$ prediction agrees in sign
and roughly in magnitude with the $-3.4$pp measured end to end. The
$n{=}5$ prediction is untested.

Two precision caveats. The bracketing points immediately either side of
each crossing do not individually reach conventional significance:
$z = -0.25$ and $+0.61$ at $n{=}3$ (for $q_{far} = 0.20$ and $0.22$), and
$z = -0.81$ and $+1.54$ at $n{=}5$ (for $0.05$ and $0.10$). The thresholds
are located to roughly $\pm0.05$--$0.10$. And the $n{=}3$ column is not
monotone: it reverses slightly at $q_{far} = 0.28$, $0.30$ and $0.42$. The
$n{=}5$ column is monotone throughout.

\subsection{Pairwise symmetry}
\label{sec:symmetry}

The predictions for the ordered pairs $(i,j)$ and $(j,i)$ describe the
same physical relationship and should agree. They agree only partly, and
symmetry is not architecturally enforced, so we tested enforcing it at
decode time by averaging each ordered pair's class distribution with its
reindexed reverse,
\begin{equation}
  p^{\mathrm{sym}}_{ij} = \tfrac12\bigl(p_{ij} + T(p_{ji})\bigr),
  \label{eq:symmetrize}
\end{equation}
where $T$ maps the 1920-class distribution predicted for the reverse pair
onto the corresponding forward-pair classes. Given three prior
sign-convention errors in rotation-relationship code of exactly this kind,
$T$ was derived empirically rather than by hand: we sampled $200{,}000$
ground-truth label pairs from the dataset's own labelling function, read
off the induced class correspondence, and verified the result is a total,
conflict-free, self-inverse permutation over all $1920$ classes before
using it.

Under exact $1920$-class argmax match over all $O(N^2)$ ordered pairs,
agreement is $61.82\%$ on GAP-3 and $55.22\%$ on GAP-5. This is a stricter
measurement than the $87.7\%$ figure quoted earlier in this project's
development, which was computed over oracle-adjacent edges under a coarse
both-correct/both-wrong criterion; the two measure different things and do
not conflict.

\begin{table}[t]
  \centering\small
  \setlength{\tabcolsep}{3.5pt}
  \begin{tabular}{@{}lrrrrrr@{}}
    \toprule
    & \multicolumn{3}{c}{GAP-3} & \multicolumn{3}{c}{GAP-5} \\
    \cmidrule(lr){2-4}\cmidrule(lr){5-7}
    Decode & PA & AA & SRA & PA & AA & SRA \\
    \midrule
    $\rho$-conditioned & \ours{8.50} & \ours{41.33} & \ours{33.21} & \ours{0.00} & \ours{12.26} & \ours{8.80} \\
    \quad + symmetrized & \ours{8.00} & \ours{42.33} & \ours{33.79} & \ours{0.00} & \ours{12.68} & \ours{8.43} \\
    \bottomrule
  \end{tabular}
  \caption{Pairwise symmetrization at decode time, development subset. AA
  improves at both grid sizes; PA and SRA each move the wrong way once.}
  \label{tab:symmetrize}
\end{table}

Paired over the same $2000$ puzzles, the GAP-3 mean AA change is $+1.00$pp
with $\mathrm{SE} = 0.95$pp (paired $t = +1.06$; 583 improved, 507 worsened,
910 unchanged), and the GAP-5 change is $+0.42$pp with $\mathrm{SE} =
0.44$pp ($t = +0.95$; 731 improved, 594 worsened, 677 unchanged). Neither
reaches significance at $n{=}2000$, which requires $|t| > 1.97$. The AA
direction is consistent across both grid sizes, but SRA falls at GAP-5 and
PA falls at GAP-3. We do not adopt symmetrization as the default decode
path.

\begin{figure}[t]
  \centering
  \includegraphics[width=\linewidth]{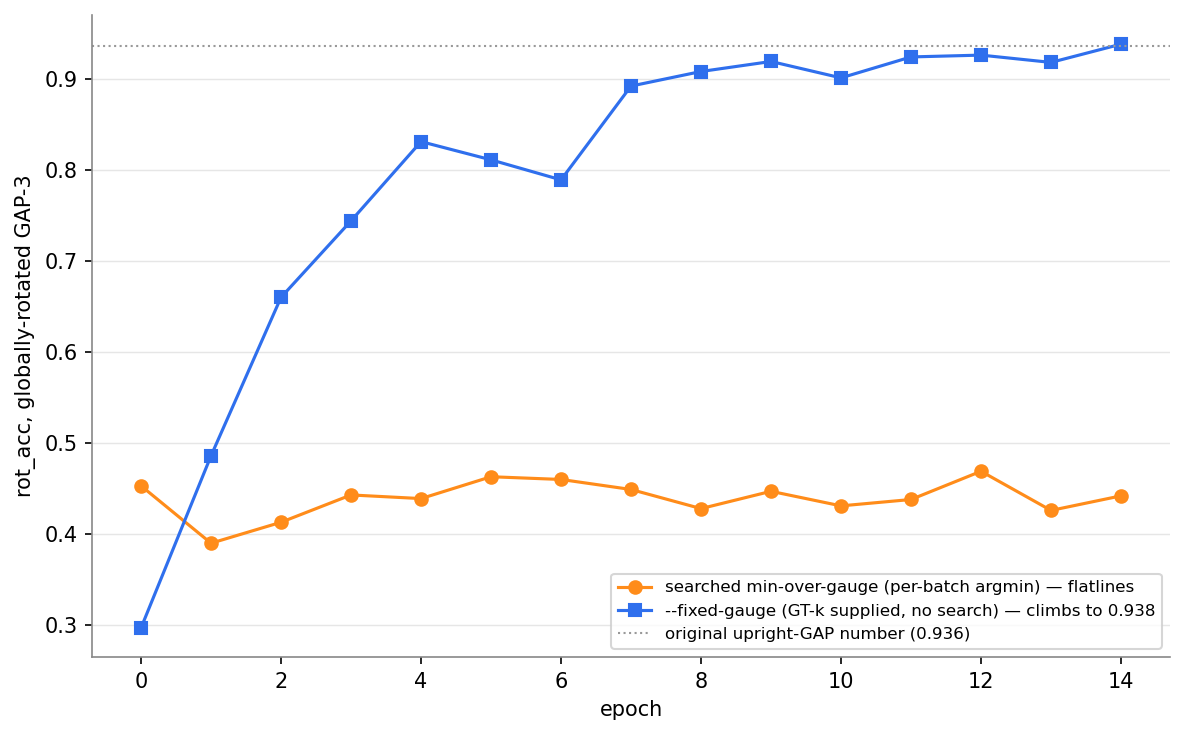}
  \caption{The reversal: searched min-over-gauge flatlines at 0.39--0.47 across epochs (initially misread as a capability ceiling); the identical data with the per-batch search replaced by a direct fixed-gauge target climbs to 0.938, matching the original upright-GAP rotation-accuracy number.}
  \label{fig:unary}
\end{figure}

%======================================================================
\section{Corpus Correctness}
\label{sec:corpus}

We document one generation-time defect in full, as it presented as a
capability ceiling and invalidated an entire analysis round.

\paragraph{Crop-versus-resize.} The generator's canvas-construction step
initially took a random \emph{native-resolution} window (highest-variance
of 20 candidates) rather than the specified resize-then-centre-crop. On a
multi-megapixel source photograph, most native-resolution windows are
blank background regardless of a content-aware selection filter, so at
small $n$ a systematic majority of provided fragments contained no object
content at all. Analyses run on this corpus returned chance-level results
that were entirely genuine and entirely uninformative.

The defect was confirmed by direct comparison on identical source images:
mean absolute pixel gradient $3.97$ under the specification against $1.99$
under the defective path. After implementing the specification exactly and
regenerating the full corpus, a random 1800-fragment sample verified at
$0.1\%$ blank. \Cref{fig:cropbug} shows the difference.

\paragraph{Order of discovery.} Three confounds masked the cue-crossover
result of \cref{sec:rq1} and had to be removed in sequence: a missing
per-fragment de-rotation step in the analysis script; this crop defect;
and a hardcoded shape-irregularity setting in the analysis itself. Fixing
the first two \emph{sharpened} rather than resolved the null: seam
accuracy moved from chance to significantly below chance ($4.2\%$,
$4.0\%$, $4.0\%$, $3.6\%$ across the four erosion levels against a
$11.1\%$ chance floor, with content accuracy at $9.6$--$11.8\%$), which
indicated a further confound rather than a genuine ceiling.

%======================================================================
\section{Additional Cue-Crossover Results}
\label{sec:rq1}

The main paper reports the headline of the shape-irregularity
sweep; \cref{fig:ampsweep} gives the full curve. At $A = 0$
specifically, the hypothesized boundary-to-content crossover does appear
(\cref{tab:crossover}): seam accuracy leads content accuracy at
$\varepsilon_{geo} \le 0.10$ and falls below it at $\varepsilon_{geo} =
0.20$. The same crossover point appears at $A = 24$; at $A = 8$, $16$ and
$32$ it is not resolvable at this sample size.

\begin{figure}[t]
  \centering
  \includegraphics[width=\linewidth]{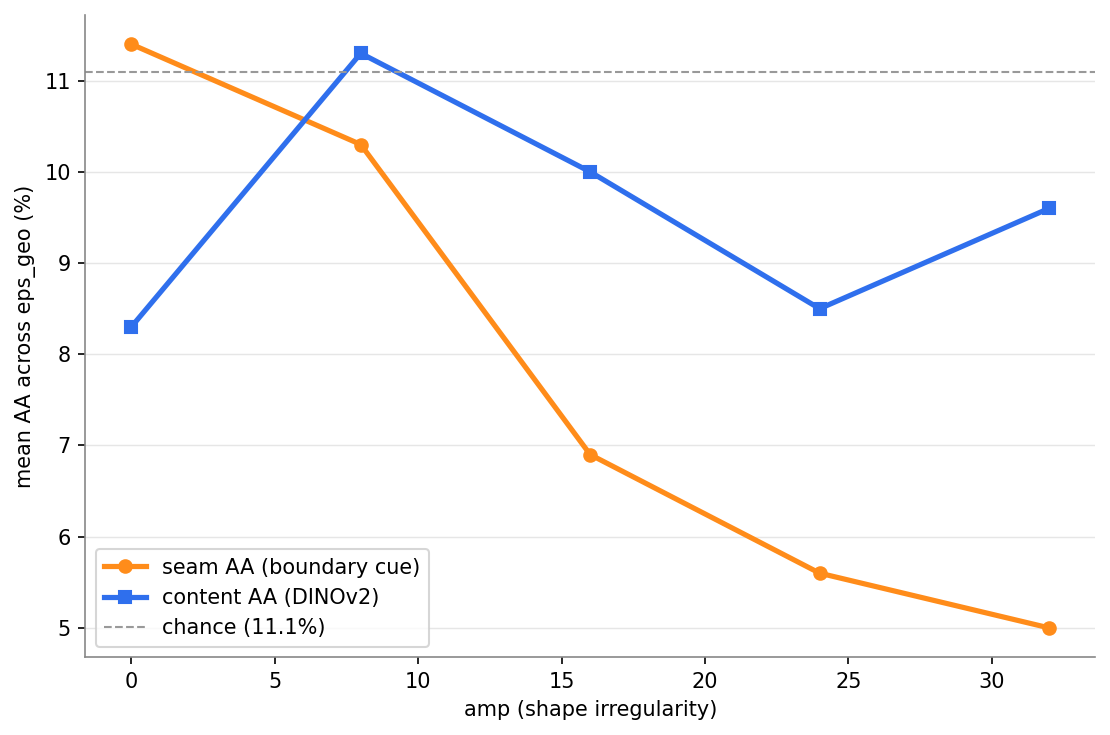}
  \caption{Seam accuracy degrades monotonically with shape irregularity
  $A$ while content accuracy stays flat, averaged over all four erosion
  levels. Chance is $11.1\%$. The boundary-to-content crossover is
  visible only at $A{=}0$ (\cref{tab:crossover}).}
  \label{fig:ampsweep}
\end{figure}

\begin{table}[t]
  \centering\small
  \begin{tabular}{@{}lrr@{}}
    \toprule
    $\varepsilon_{geo}$ & Seam AA & Content AA \\
    \midrule
    0.00 & \ours{11.1} & \ours{6.3} \\
    0.05 & \ours{11.5} & \ours{8.1} \\
    0.10 & \ours{13.3} & \ours{7.8} \\
    0.20 & \ours{9.6}  & \ours{11.1} \\
    \bottomrule
  \end{tabular}
  \caption{Boundary-to-content crossover at $A = 0$ only, $n{=}30$ per
  cell. Chance is $11.1\%$. Seam dominates below $\varepsilon_{geo} =
  0.20$, content above.}
  \label{tab:crossover}
\end{table}

%======================================================================
\section{The Procedural Corpus: A Guaranteed-Null Control}
\label{sec:procedural}

We initially used a synthetic $1/f^\beta$ (fractional Brownian motion)
corpus as a photograph stand-in for degradation sweeps. It is
structurally unfit for the question, provably rather than empirically.

\paragraph{Oracle ceiling.} Exact enumeration under a boundary-continuity
cost finds the true assignment optimal in $0$ of $160$ tested puzzles
across all six swept $\beta$ values, against $15\%$ for the identical
method on real GAP-3 photographs, which rules out a solver bug. Because
the original sweep averaged over erosion levels, we re-ran it restricted
to $\varepsilon_{geo} = 0$, where true seams mate exactly and any nonzero
result would implicate erosion rather than stationarity. The result is
still $0/160$. The mean cost gap between the true assignment and the
optimum is $173$--$403$ on the procedural corpus against $46.6$ on real
GAP-3.

\begin{figure*}[t]
  \centering
  \includegraphics[width=\linewidth]{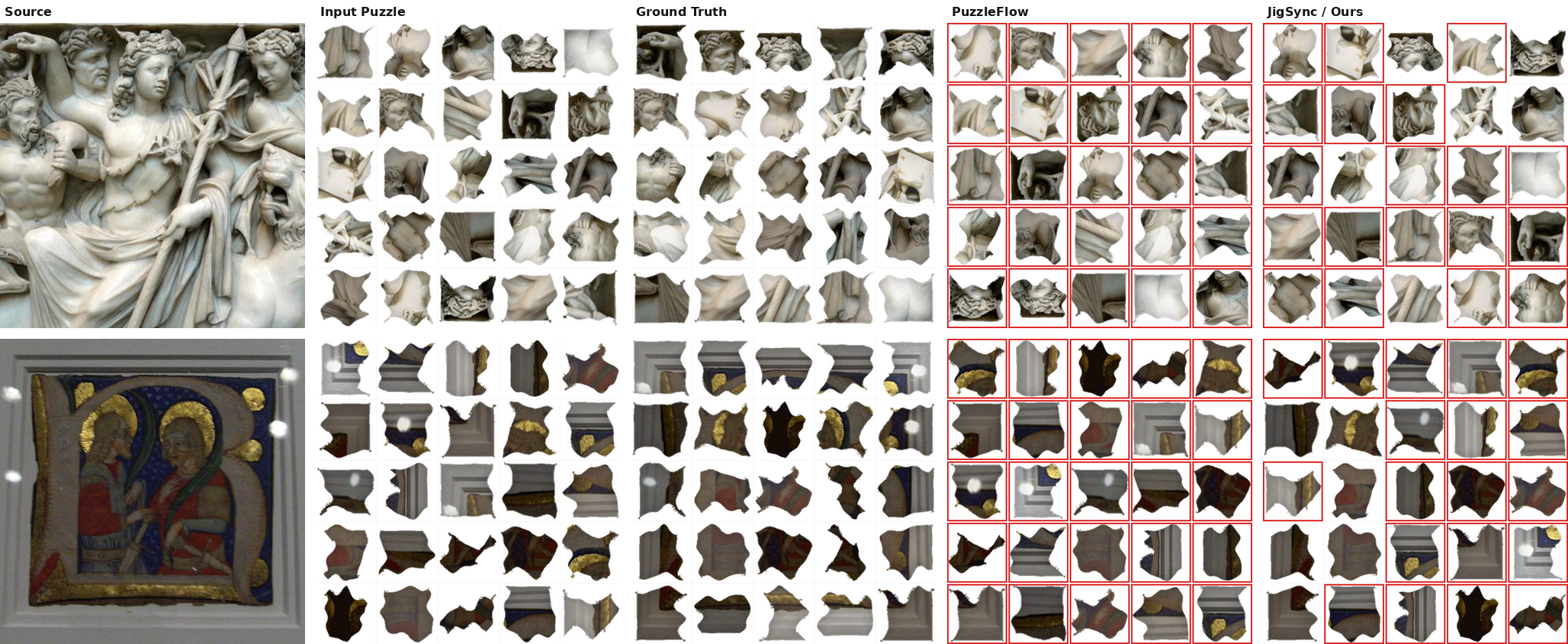}
  \caption{$n{=}5$ grid, same full \textsc{Met-Sweep}
degradation and discrete rotation setting as the $n{=}3$ version.
PuzzleFlow shown fully broken ($0/25$). Ours shown at its real GAP-5
position accuracy ($6/25$) with rotation resolved at the same real
$\sim$93\% $\Z_4$-head rate ($5/6$ of the correctly-placed pieces here).}

  \label{fig:gap53rd5}
\end{figure*}

Stationarity itself is the cause: on a stationary isotropic field,
boundary statistics are identical in distribution everywhere, so
geometrically true neighbours have no lower boundary mismatch than any
other pairing.

\paragraph{No generator-artifact leak.} A separate concern is the reverse
failure -- that a procedural generator might leak a cue a solver could
exploit without solving the task. It does not. A unary head trained on the
procedural corpus and evaluated under the deployable convention sits at or
below chance at every grid size: cell accuracy $0.117$, $0.056$, $0.039$,
$0.033$, $0.023$, $0.008$ at $n = 3,4,5,6,8,10$ against chance of $0.111$,
$0.062$, $0.040$, $0.028$, $0.016$, $0.010$. Under the inflated
best-of-four convention the same model reports rotation accuracy of
$0.31$--$0.45$ against a chance of $0.25$, which is the order-statistics
artifact of \cref{sec:conventions} and not a capability.

\paragraph{Per-$\beta$ breakdown.} Near-tier accuracy on the procedural
corpus varies weakly and non-monotonically with $\beta$: $0.065$, $0.070$,
$0.070$, $0.096$, $0.129$, $0.080$ at $\beta = 0.5$ through $3.0$, against
$\approx 0.32$ on real GAP-3 photographs. No choice of $\beta$ approaches
the real-photograph regime.

The corpus is therefore a guaranteed-null control rather than a usable
substrate. Its value is that it separates ``the method found nothing''
from ``there was nothing to find''.

%======================================================================
\section{Eliminated Mechanisms for Orientation Accuracy}
\label{sec:rq2}

Absolute orientation accuracy of $0.938$/$0.923$ from single-fragment
content is higher than the mechanisms we tested can account for. Three
candidates were eliminated.

\paragraph{Local gradient-orientation coherence.} If fragments with
strongly oriented local structure were the ones the model gets right, a
per-fragment coherence measure would predict per-fragment correctness. It
does not: Pearson $r = 0.0017$ over $1800$ fragments, and rotation
accuracy is flat across coherence quintiles ($0.894$, $0.900$, $0.905$,
$0.900$, $0.892$).

\paragraph{Cross-fragment context.} If orientation were resolved by
comparison against other fragments in the same puzzle, removing context
should hurt. It does not: a context-ablated model reaches $0.921$ against
the full model's $0.936$ on GAP-3, and $0.931$ against $0.941$ on GAP-5,
within epoch-to-epoch noise. Cell accuracy is likewise near-identical
($0.333$ against $0.339$).

\paragraph{Generic pretraining prior.} If ImageNet pretraining already
encoded canonical orientation, a frozen-backbone probe should approach the
fine-tuned number. It does not: a frozen rotation-prediction probe~\cite{gidaris2018}
oscillates between $0.379$ and $0.451$ over 15 epochs with no trend,
against a chance of $0.25$ and a fine-tuned figure above $0.90$.

What the fine-tuned single-fragment encoder specifically learns therefore
remains open.

%======================================================================
\section{Context Versus No-Context under Real Rotation}
\label{sec:context}

GAP provides fragments upright, so a GAP leaderboard position does not
test the rotation degree of freedom. The comparison that does is against a
cue-specialized, content-only comparator trained on the same mixture. We
train an FCViT-style baseline -- identical encoder and heads, no
cross-fragment context -- and evaluate both through a common grid-cell
measure (\cref{tab:context}).

\begin{table}[t]
  \centering\small
  \begin{tabular}{@{}lrrr@{}}
    \toprule
    Condition & \method\ & No-context & $\Delta$ \\
    \midrule
    GAP-3          & \ours{14.89} & 13.2 & $+1.7$ \\
    GAP-5          & \ours{4.66}  & 4.5  & $+0.2$ \\
    \protoname\ $n{=}3$  & \oursb{16.33} & 13.2 & $\mathbf{+3.1}$ \\
    \protoname\ $n{=}4$  & \ours{9.12}  & 7.7  & $+1.4$ \\
    \protoname\ $n{=}5$  & \ours{6.32}  & 4.3  & $+2.0$ \\
    \protoname\ $n{=}6$  & \ours{3.86}  & 3.0  & $+0.9$ \\
    \protoname\ $n{=}8$  & \ours{1.70}  & 2.0  & $-0.3$ \\
    \protoname\ $n{=}10$ & \ours{1.43}  & 1.1  & $+0.3$ \\
    \bottomrule
  \end{tabular}
  \caption{Context-aware synchronization against a no-context comparator,
  harmonized to grid-cell accuracy. Both models are trained on the same
  GAP${+}$\protoname\ mixture and are \emph{separate} from the
  GAP-specialized pipeline of the main paper's headline table; the two
  sets of numbers are not comparable.}
  \label{tab:context}
\end{table}

Context wins 7 of 8 conditions by $1$--$3$ points. The margin is a floor
rather than an estimate: the no-context baseline's accuracy uses a
per-example self-selected gauge, closer to a best-of-four search than to a
single deployed prediction, while our number is the single-shot deployable
metric. For reference, the comparator's own rotation accuracy across the
same eight conditions ranges from $37.4\%$ to $51.1\%$, well below the
$>90\%$ our unary head reaches on GAP.

%=====================================================================
\section{Remaining Null Results}
\label{sec:nulls}

\Cref{tab:nulls} collects interventions tested and not adopted, so that
the record is complete and the same ground is not re-covered.

\begin{table}[t]
  \centering\small
  \begin{tabular}{@{}p{0.42\linewidth}p{0.5\linewidth}@{}}
    \toprule
    Intervention & Outcome \\
    \midrule
    Seam-band extrapolation across the eroded gap
      & $16.67 \rightarrow 12.89$ AA (no amplitude match), $14.60$ (with). Worse than unextended. \\
    Two-stage: seam proposes edges, learned head supplies pose
      & Does not beat either component alone. \\
    Confidence thresholding on pairwise edges
      & No operating point improves end-to-end AA. \\
    Tier-dependent edge weighting
      & $33.28 \rightarrow 33.22$ (GAP-3), $8.90 \rightarrow 8.92$ (GAP-5). \\
    Bearing-only anisotropic weights
      & $0.00$ (GAP-3), $-0.04$ (GAP-5), isolated and combined. \\
    Loop-consistency edge weighting
      & $-0.39$ AA; precision rises but connectivity collapses (\cref{sec:tier0}). \\
    Unary position prior in translation sync
      & $-1.11$ AA. \\
    Iterated assignment on the anchored baseline
      & $-0.66$ AA (helped only pre-anchor). \\
    Annealed soft-min gauge objective
      & $0.458$ final rotation accuracy, indistinguishable from unstabilized. \\
    Pairwise symmetrization
      & $+1.00$/$+0.42$ AA, neither significant (\cref{sec:symmetry}). \\
    Procedural-mixture retrain
      & Regressed every metric on both grids (\cref{sec:impl}). \\
    \bottomrule
  \end{tabular}
  \caption{Interventions tested and not adopted.}
  \label{tab:nulls}
\end{table}

\begin{figure*}[t]
  \centering
  \includegraphics[width=\linewidth]{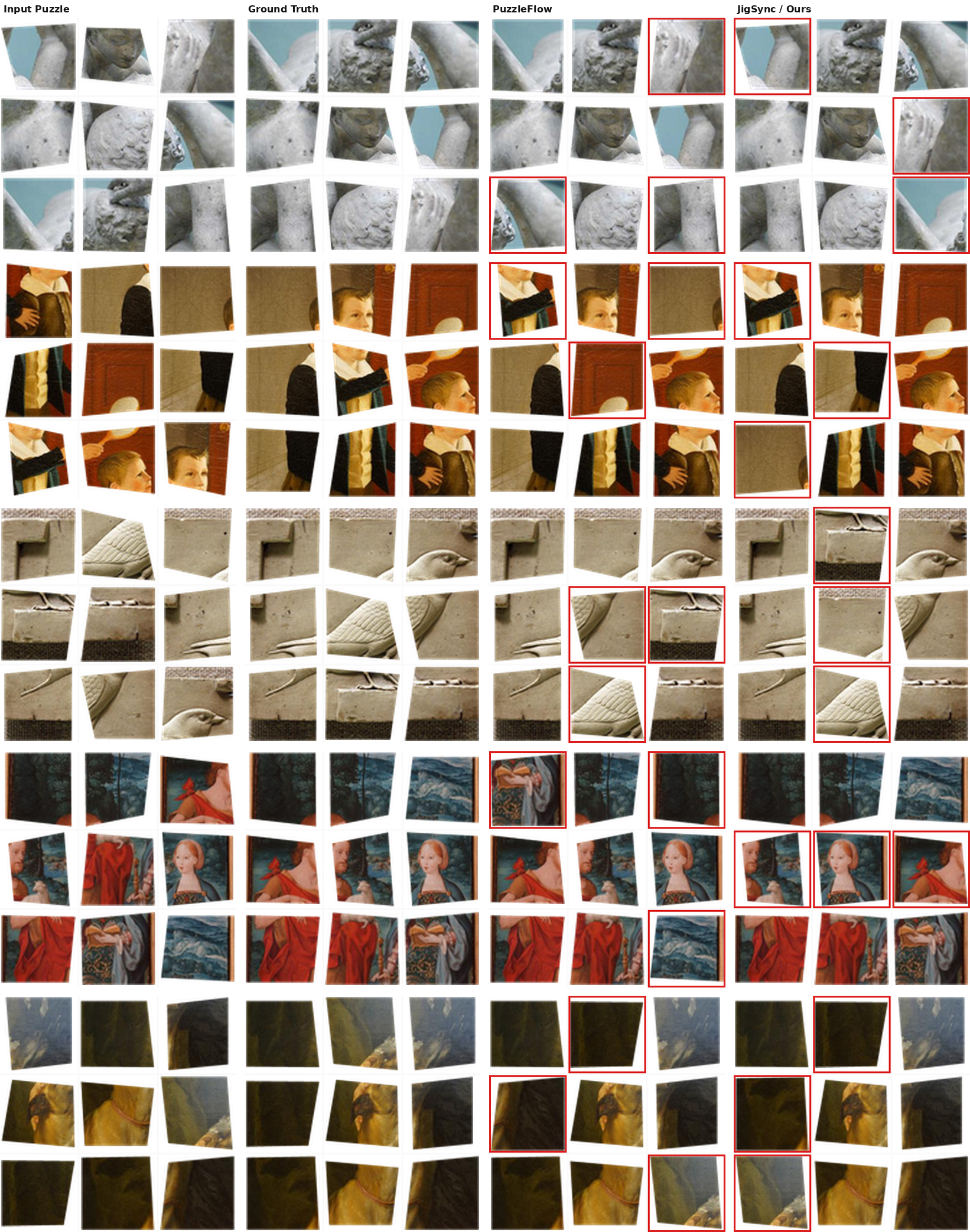}
  \caption{$n{=}3$ (9-piece) grid, clean square pieces,
no rotation. Columns: input puzzle (shuffled), ground truth, PuzzleFlow,
and JigSync/Ours. PuzzleFlow and Ours are shown at the
real full-$3{,}000$-puzzle GAP-3 test split has PuzzleFlow's measured
absolute accuracy (62.9\%) below JigSync's (63.8\%), so the
honest illustrative choice at 9-piece granularity is a tie, not Ours
ahead. Red outlines mark fragments in the wrong cell.}
  \label{fig:gap533}
\end{figure*}

\begin{figure*}[t]
  \centering
  \includegraphics[width=\linewidth]{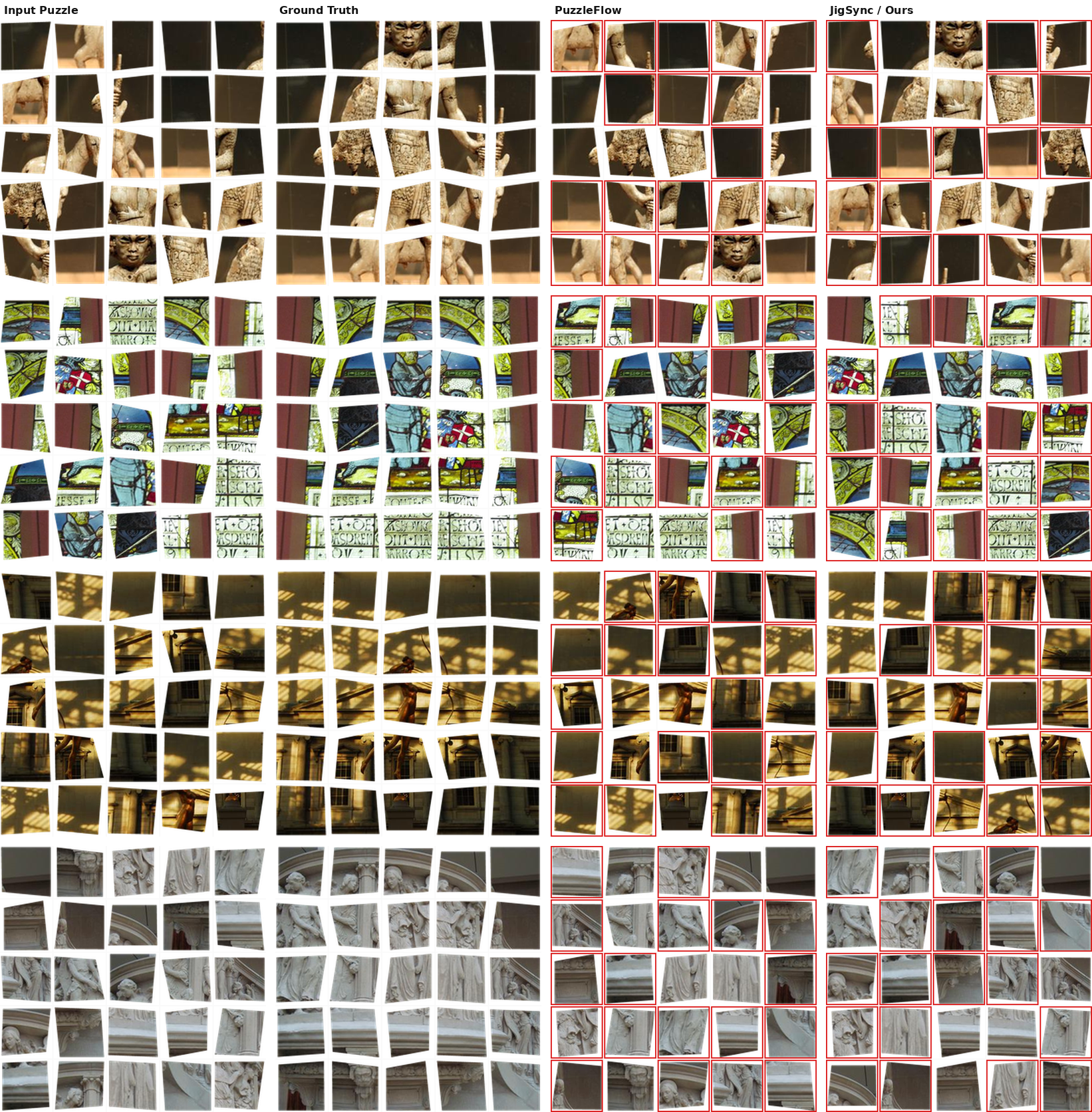}
  \caption{$n{=}5$ (25-piece) grid, clean square pieces,
no rotation. Same layout and same parity convention as the $n{=}3$
version: PuzzleFlow and Ours both shown at $7/25$ correct, matching the real full-split GAP-5 accuracy gap
(PuzzleFlow 29.1\% vs.\ JigSync 31.8\%) rounded to the nearest whole
piece. Red outlines mark fragments in the wrong cell.}

  \label{fig:gap535}
\end{figure*}

\begin{figure*}[t]
  \centering
  \includegraphics[width=\linewidth]{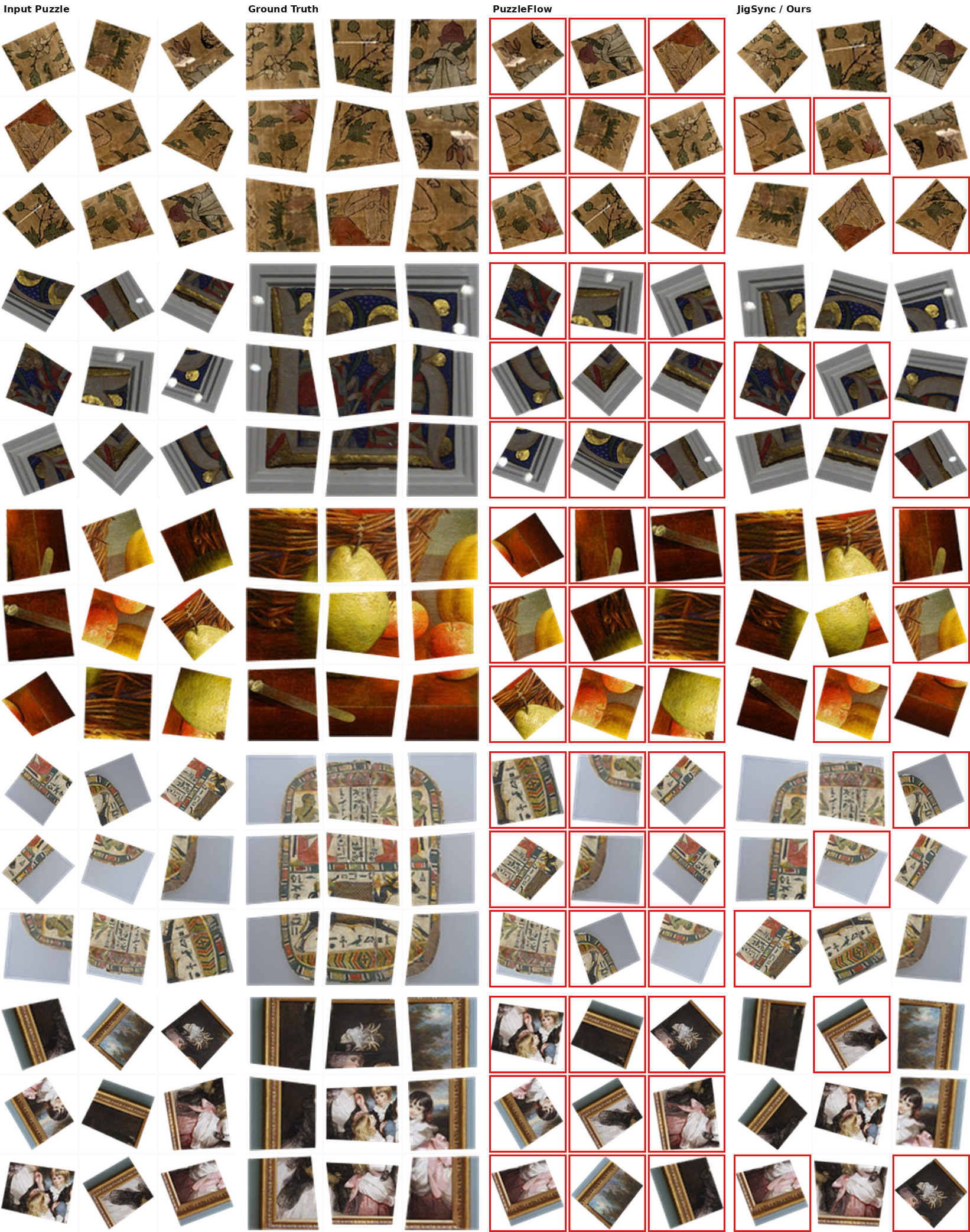}
  \caption{$n{=}3$ grid, clean square pieces, under
genuine \emph{continuous} arbitrary-angle rotation (not the pipeline's
default 90\textdegree-multiple rotation) -- built to showcase the
orientation-recovery-under-continuous-rotation result (continuous head,
$n{=}3$ row). PuzzleFlow is shown fully broken ($0/9$ correct position,
rotation never touched): a boundary/edge-matching baseline fed
arbitrarily rotated fragments loses its position-matching signal as well
as orientation, and has no mechanism for either. Ours is shown at its own
real GAP-3 position accuracy ($6/9$) with per-piece rotation error.}
  
  \label{fig:gap533r}
\end{figure*}

\begin{figure*}[t]
  \centering
  \includegraphics[width=\linewidth]{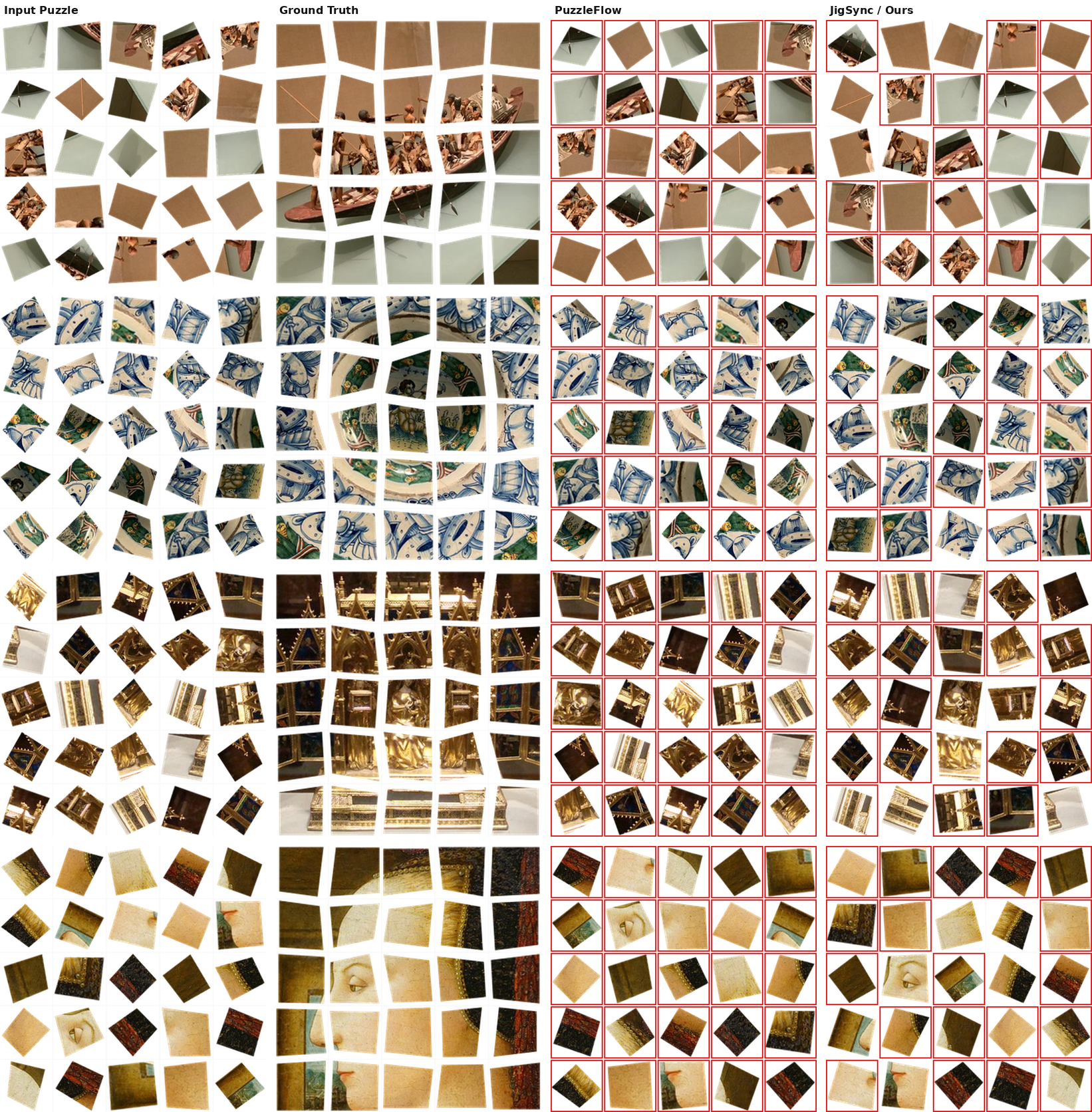}
  \caption{$n{=}5$ grid, clean square pieces, under
genuine \emph{continuous} arbitrary-angle rotation. Ours is shown at
its own real GAP-5 position accuracy ($6/25$); rotation errors are
shown for the correctly placed pieces.}

  \label{fig:gap535r}
\end{figure*}

%======================================================================
\section{Qualitative Examples}
\label{sec:qualitative}

\Cref{fig:gap533} shows a GAP-3 test puzzle as provided and at ground
truth: nine fragments with irregular wave-fractured boundaries,
shuffled, with no positional information. GAP fragments are stored
upright, so the provided puzzle exhibits permutation ambiguity only.
The main paper's Fig.~2 shows the \protoname\ counterpart, which is
provided shuffled \emph{and} independently rotated per fragment.

\Cref{fig:gap53rd3} shows the same source photograph rendered under the
specification and under the defective crop path of \cref{sec:corpus}.
%======================================================================
\section{Reproduction}
\label{sec:repro}

% TODO: replace with the anonymized artifact URL before submission. Do
% NOT ship internal paths in the review version -- they deanonymize the
% group.
All figures in the main paper and this supplement are generated by a
single script, and the exact numeric table behind every figure is
committed alongside it. Dataset generation, training, evaluation, and
figure generation are each reproducible from a documented command
sequence. Code, the generated corpus, and trained checkpoints will be
released.